\documentclass[twoside,11pt]{article}
\usepackage{jmlr2e_arxiv}

\usepackage[latin1]{inputenc}
\usepackage{amssymb,amsmath,array}
\usepackage{thmtools,thm-restate}
\usepackage{graphicx}
\usepackage{caption}
\usepackage{subcaption}

\usepackage{booktabs}
\usepackage{multirow}
\usepackage{siunitx}
\usepackage{multicol}
\usepackage{array}
\usepackage[table,xcdraw]{xcolor}
\usepackage[figuresright]{rotating}

\usepackage{natbib}

\usepackage{algorithm}
\usepackage{algorithmicx}
\usepackage{algpseudocode}
\usepackage{enumerate}

\usepackage{tikz}
\usepackage{calc}
\usetikzlibrary{arrows, shapes, positioning}
\usepackage{hyperref}
	
\newcommand{\ind}{\perp}
\newcommand{\dep}{\not\perp}
\newcommand{\condind}[3]{(#1 {\ind} #2 \ | \ #3)}
\newcommand{\conddep}[3]{(#1 {\dep} #2 \ | \ #3)}

\newcommand{\FBS}[0]{FBS}
\newcommand{\fFBS}[1]{FBED$^#1$}

\newcommand{\datasetref}[8]{\multirow{2}{*}{#1} & \multirow{2}{*}{#2} & \multirow{2}{*}{#3} & \multirow{2}{*}{#4} & \multirow{2}{*}{#5} & \multirow{2}{*}{#6} & #7 \\
& & & & & & #8 \\}

\DeclareMathOperator*{\argmax}{argmax}

\algrenewcomment[1]{//\textit{#1}}

\newcommand{\algrule}[1][.2pt]{\par\vskip.5\baselineskip\hrule height #1\par\vskip.5\baselineskip}%

\begin{document}

\title{Forward-Backward Selection with Early Dropping}

\author{\name Giorgos Borboudakis \email borbudak@gmail.com \\
       \addr Computer Science Department, University of Crete\\
       Gnosis Data Analysis IKE
       \AND
       \name Ioannis Tsamardinos \email tsamard.it@gmail.com \\
       \addr Computer Science Department, University of Crete\\
       Gnosis Data Analysis IKE}
       
\editor{}
\maketitle

\begin{abstract}
Forward-backward selection is one of the most basic and commonly-used feature selection algorithms available. 
It is also general and conceptually applicable to many different types of data. 
In this paper, we propose a heuristic that significantly improves its running time, while preserving predictive accuracy.
The idea is to temporarily discard the variables that are conditionally independent with the outcome given the selected variable set. 
Depending on how those variables are reconsidered and reintroduced, this heuristic gives rise to a family of algorithms with increasingly stronger theoretical guarantees.
In distributions that can be faithfully represented by Bayesian networks or maximal ancestral graphs, members of this algorithmic family are able to correctly identify the Markov blanket in the sample limit.
In experiments we show that the proposed heuristic increases computational efficiency by about two orders of magnitude in high-dimensional problems, while selecting fewer variables and retaining predictive performance.
Furthermore, we show that the proposed algorithm and feature selection with LASSO perform similarly when restricted to select the same number of variables, making the proposed algorithm an attractive alternative for problems where no (efficient) algorithm for LASSO exists.
\end{abstract}

\section{Introduction}

The problem of feature selection (a.k.a. variable selection) in supervised learning tasks can be defined as the problem of selecting a minimal-size subset of the variables that leads to an optimal, multivariate predictive model for a target variable (outcome) of interest \citep{Tsamardinos2003}. 
Thus, the feature selection's task is to filter out irrelevant variables and variables that are superfluous given the selected ones (that is, weakly relevant variables, see \citep{John94, Tsamardinos2003}).

Solving the feature selection problem has several advantages. 
Arguably, the most important one is knowledge discovery: by removing superfluous variables it improves intuition and understanding about the data-generating mechanisms. 
This is no accident as solving the feature selection problem has been linked to the data-generating causal network \citep{Tsamardinos2003}. 
In fact, {\em it is often the case that the primary goal of data analysis is feature selection} and not the actual resulting predictive model. 
This is particularly true in medicine and biology where the features selected may direct future experiments and studies. 
Feature selection is also employed to reduce the cost of measuring the features to make operational a predictive model; for example, it can reduce the monetary cost or inconvenience to a patient of applying a diagnostic model by reducing the number of medical tests and measurements required to perform on a subject for providing a diagnosis. 
Feature selection also often improves the predictive performance of the resulting model in practice, especially in high-dimensional settings. 
This is because a good-quality selection of features facilitates modeling, particularly for algorithms susceptible to the curse of dimensionality.
There has been a lot of research on feature selection methods in the statistical and machine learning literature.
An introduction to the topic, as well as a review of many, prominent methods can be found in \citep{Guyon2003}, while the connections between feature selection, the concept of relevancy, and probabilistic graphical models is in \citep{John94, Tsamardinos2003}. 

We will focus on forward and backward selection algorithms, which are specific instances of stepwise methods \citep{AppliedLinearStatisticalModels, WeisbergAppliedLinearRegression}. 
These methods are some of the oldest, simplest and most commonly employed feature selection methods. 
In computer science they have re-appeared in the context of Markov blanket discovery and Bayesian network learning \citep{Margaritis2000, Tsamardinos2003IAMB, Margaritis2009}. 
An attractive property of stepwise methods is that they are very general, and are applicable to different types of data.
For instance, stepwise methods using conditional independence tests or information criteria can be directly applied to (a) mixed continuous and categorical predictors, (b) cross-sectional or time course data, (c) continuous, nominal, ordinal or time-to-event outcomes, among others, (d) with non-linear tests, such as kernel-based methods \citep{Zhang2011}, and (e) to heteroscedastic data using robust tests; many of the aforementioned tests, along with others have been implemented in the MXM R package \citep{Lagani2016}.

Forward selection has several issues. 
First, it can be relatively slow, performing $O(p k)$ tests for variable inclusion, where $p$ and $k$ are the total number of variables and the number of selected variables respectively.
This is acceptable for low-dimensional datasets, but becomes unmanageable with increasing dimensionality.
Second, forward selection suffers from multiple testing problems and thus may select a large number of irrelevant variables \citep{Flom2007}.

In this work we extend the forward selection algorithm to deal with the problems above. 
In Section~\ref{sec:impr_fs} we propose a heuristic to reduce its computational cost without sacrificing quality, while also selecting fewer variables and reducing multiple testing issues.
The idea is, in each iteration of the forward search, to filter out variables that are deemed conditionally independent of the target given the current set of selected variables. 
After termination, the algorithm is allowed to run up to $K$ additional times, every time initializing the set of selected variables to the ones selected in the previous run.
Finally, backward selection is applied on the selected variables.
We call this algorithm \textbf{Forward-Backward selection with Early Dropping} (\fFBS{K}). 
This heuristic is inspired by the theory of Bayesian networks and maximal ancestral graphs \citep{Spirtes2000, SpirtesRichardson2002}, and similar ideas have been successfully applied by other feature selection methods \citep{Aliferis2010JMLR}.
In Section~\ref{sec:mbfbs} we show that (a) \fFBS{0} returns a superset of the adjacent nodes in any Bayesian network or maximal ancestral graph that faithfully represents the data distribution (if there exists one and assuming perfect statistical independence tests), (b) \fFBS{1} returns the Markov blanket of the data distribution, provided the distribution is faithful to a Bayesian network, and (c) \fFBS{\infty} returns the Markov blanket of the data distribution provided the distribution is faithful to a maximal ancestral graph, or equivalently, it is faithful to a Bayesian network where some variables are unobserved (latent).
In the experimental evaluation presented in Section~\ref{sec:exp:ffbs}, we show that \fFBS{0} and \fFBS{1} result in predictive models that are on par with the other methods and perform about 1-2 orders of magnitude fewer conditional independence tests than standard forward selection.
Furthermore, we show that \fFBS{K} performs equally well to feature selection with LASSO \citep{Tibshirani1996} feature selection, if both algorithms select the same number of variables.

\section{Notation and Preliminaries}
\label{sec:bg}

We start by introducing the notation and terminology used throughout the paper.
We use upper-case letters to denote single variables (for example, $X$), and bold upper-case letters to denote sets of variables (for example, $\mathbf{Z}$).
We will use $|\mathbf{X}|$ to refer to the number of elements contained in set $\mathbf{X}$.
The terms variable, feature or predictor will be used interchangeably.
We will use $p$ and $n$ to refer to the number of variables and samples in a dataset $\mathcal{D}$ respectively.
The set of variables in $\mathcal{D}$ will be denoted as $\mathbf{V_{\mathcal{D}}}$.
The target variable (also called outcome) will be referred to as $T$.
Next, we proceed with the basics about stepwise feature selection methods \citep{AppliedLinearStatisticalModels, WeisbergAppliedLinearRegression}.

\subsection{Stepwise Feature Selection}
\begin{algorithm}[ht!]
\caption{Forward-Backward Selection (\FBS{})}
\label{alg:fbs}
\begin{algorithmic}[1]
	\Require Dataset $\mathcal{D}$, Target $T$
	\Ensure Selected Variables $\mathbf{S}$
	\State $\mathbf{S} \gets \emptyset$ \Comment{\textit{Set of selected variables}}
	\State $\mathbf{R} \gets \mathbf{V_{\mathcal{D}}}$ \Comment{\textit{Set of remaining candidate variables}}
	\State
	\State \Comment{\textit{Forward phase: iterate until $\mathbf{S}$ does not change}}
	\While{$\mathbf{S}$ changes} 
	\State \Comment{\textit{Identify the best variable $V_{best}$ out of all remaining variables $\mathbf{R}$, according to \textproc{Perf}}}
	\State $V_{best}$ $\gets$ $\argmax\limits_{V \in \mathbf{R}}$ \textproc{Perf}$(\mathbf{S} \cup V)$
	\State \Comment{\textit{Select $V_{best}$ if it increases performance according to criterion \textproc{C}}}
	\If{\textproc{Perf}$(\mathbf{S} \cup V_{best}) \underset{\scriptscriptstyle{\textproc{C}}}{>} $ \textproc{Perf}$(\mathbf{S})$}
	\State $\mathbf{S}$ $\gets$ $\mathbf{S} \cup V_{best}$
	\State $\mathbf{R}$ $\gets$ $\mathbf{R} \setminus V_{best}$
	\EndIf
	\EndWhile
	\State
	\State \Comment{\textit{Backward phase: iterate until $\mathbf{S}$ does not change}}
	\While{$\mathbf{S}$ changes} 
	\State \Comment{\textit{Identify the worst variable $V_{worst}$ out of all selected variables $\mathbf{S}$, according to \textproc{Perf}}}
	\State $V_{worst}$ $\gets$ $\argmax\limits_{V \in \mathbf{S}}$ \textproc{Perf}$(\mathbf{S} \setminus V)$
	\State \Comment{\textit{Check if removing $V_{worst}$ does not decrease performance according to criterion \textproc{C}}}
	\If{\textproc{Perf}$(\mathbf{S} \setminus V_{worst}) \underset{\scriptscriptstyle{\textproc{C}}}{\geq} $ \textproc{Perf}$(\mathbf{S})$}
	\State $\mathbf{S}$ $\gets$ $\mathbf{S} \setminus V_{worst}$
	\EndIf
	\EndWhile	
	\State {\bfseries return} $\mathbf{S}$
\end{algorithmic}
\end{algorithm}

Stepwise methods start with some set of selected variables and try to improve it in a greedy fashion, by either including or excluding a single variable at each step.
There are various ways to combine those operations, leading to different members from the stepwise algorithmic family.
Two popular members of the stepwise family are the \textbf{forward selection}  and \textbf{backward selection} (also known as backward elimination) algorithms.
Forward selection starts with a (usually empty) set of variables and adds variables to it, until some stopping criterion is met.
Similarly, backward selection starts with a (usually complete) set of variables and then excludes variables from that set, again, until some stopping criterion is met.
Typically, both methods try to include or exclude the variable that offers the highest performance increase.
We will call each step of selecting (removing) a variable a forward (backward) \textbf{iteration}.
Executing forward (backward) iterations until termination will be called a forward (backward) \textbf{phase} respectively.
An instance of the stepwise family, which we focus on hereafter, is the \textbf{Forward-Backward Selection} algorithm (\FBS{}), which first performs a forward phase and then a backward phase on the selected variables.
This algorithm is not new; similar algorithms have appeared in the literature before (see \citep{Margaritis2000, Tsamardinos2003IAMB, Margaritis2009} for example). 

\FBS{} is shown in Algorithm~\ref{alg:fbs}.
The function \textproc{Perf} evaluates a set of variables and returns their performance relative to some statistical model.
Examples are the log-likelihood for logistic regression, the partial log-likelihood for Cox regression and the F-score for linear regression, or the AIC \cite{Akaike1973} or BIC \cite{Schwarz1978} penalized variants of those performance metrics.
The \textbf{selection criterion} \textproc{C} compares the performance of two sets of variables as computed by \textproc{Perf}.
For instance, in the previous example \textproc{C} could perform a likelihood ratio test and use a predetermined significance level $\alpha$ to make a decision
\footnote{
In general, the type of criteria used in practice are not limited to that. For example, one may also stop after a fixed number of variables have been selected.}; we will describe such selection criteria in the next subsection.
We will use the predicates $\underset{\scriptscriptstyle{\textproc{C}}}{>}$, $\underset{\scriptscriptstyle{\textproc{C}}}{\geq}$ and $\underset{\scriptscriptstyle{\textproc{C}}}{=}$ to compare two sets of variables; they are true if the left-hand-side value is greater, greater or equal, or equal than the right-hand-side value respectively, according to the criterion \textproc{C}.

\subsection{Criteria for Variable Selection}
\label{sec:criteria}
Next we will briefly describe some performance functions and selection criteria that are employed in practice; for more details see \citep{AppliedLinearStatisticalModels, WeisbergAppliedLinearRegression}.
The most common choices are statistical tests, information criteria and cross-validation.
We describe statistical tests and information criteria next; we did not consider cross-validation, mainly because of its high computational cost.

\subsubsection{Statistical Tests}
Since the models tested at each iteration are nested, one can employ a likelihood-ratio (LR) test (or asymptotically equivalent approximations thereof such as score tests and Wald tests) for nested models as a selection criterion.
We next describe the likelihood-ratio test in more depth.
For the LR test, the performance \textproc{Perf} is related to the log-likelihood (\textproc{LL}) and the criterion \textproc{C} tests the hypothesis that both models are equivalent with respect to some pre-specified significance level $\alpha$.
Let \textproc{Dev}$(T|\mathbf{X}) \equiv -2 \cdot \textproc{LL}(T|\mathbf{X})$ and $\textproc{Par}(T|\mathbf{X})$ be the deviance and number of parameters respectively of a model for target $T$ using variables $\mathbf{X}$.
Then, the statistic $\mathit{Stat}$ of a nested test for models with variables $\mathbf{X}$ (null model) and $\mathbf{X} \cup \mathbf{Y}$ (alternative model) is computed as the difference in deviance of both models, that is, $\mathit{Stat} \equiv \textproc{Dev}(T|\mathbf{X}) - \textproc{Dev}(T|\mathbf{X} \cup \mathbf{Y})$, and follows asymptotically a $\chi^2$ distribution with $\textproc{Par}(T|\mathbf{X} \cup \mathbf{Y}) - \textproc{Par}(T|\mathbf{X})$ degrees of freedom \citep{Wilks1938}
\footnote{
This result assumes that the larger hypothesis is correctly specified.
In case of model misspecification, the statistic follows a different distribution \citep{Foutz1977}.
Methods to handle model misspecification have been proposed by \cite{White1982} and \cite{Vuong1989}. 
A method for dealing with model misspecification in model selection with information criterion is presented in \citep{Jinchi2014}.
As this problem is out of this paper's scope, we did not further consider it.
}.

Tests for nested models are essentially \textbf{conditional independence tests}, relative to some statistical model (for example, using linear regression without interaction terms tests for linear dependence), and assuming that the model is correctly specified.
If the null model contains the predictors $\mathbf{X}$ and the alternative model contains $\mathbf{X} \cup \mathbf{Y}$, the nested test tests the hypothesis that the coefficients of $\mathbf{Y}$ are zero, or equivalently, that $\mathbf{Y}$ is conditionally independent of the target $T$ given $\mathbf{X}$.
We denote \textbf{conditional independence} of two non-empty sets $\mathbf{X}$ and $\mathbf{Y}$ given a (possibly empty) set $\mathbf{Z}$ as $\condind{\mathbf{X}}{\mathbf{Y}}{\mathbf{Z}}$.
Finally, we note that one is not limited to likelihood-ratio based conditional independence tests, but can use any appropriate conditional independence test, such as a kernel-based test \citep{Zhang2011}.

A problem when using statistical tests for feature selection is that, due to multiple testing, the test statistics do not have the claimed distribution \citep{ElementsOfStatisticalLearning2009} and the resulting p-values are too small \citep{RegressionModellingStrategies2001, Flom2007}, leading to a high false discovery rate.
However, they are still very useful tools for the task of variable selection, if used with care.
In case one is interested in the resulting model, forward selection is sub-optimal, as the resulting model will have inflated coefficients due to the feature selection procedure \citep{Flom2007}, reducing its predictive ability.
Instead, methods performing regularization (like L1, L2 or elastic net) are more appropriate.
In any case, a method like cross-validation should be used to estimate out-of-sample predictive performance of the final model.
There have been various approaches to deal with multiple testing, like correcting the p-values using resampling methods \citep{Finos2010} or dynamically adjusting significance levels \citep{Hwang2015}.
We will not consider the previously referred methods in this paper; we note however that our proposed method is orthogonal to those methods and could be used in conjunction with them.

\subsubsection{Information Criteria}\label{sec:ic}
Another way to compare two (or more) competing models is to use information criteria, such as the Akaike information criterion (AIC) \citep{Akaike1973} or the Bayesian information criterion (BIC) \citep{Schwarz1978}.
Information criteria are based on the fit of a model but additionally penalize the model by its complexity.
The AIC and BIC scores of a model for $T$ based on $\mathbf{X}$ are defined as follows:
\begin{align*}
\textproc{AIC}(T|\mathbf{X}) & \equiv \textproc{Dev}(T|\mathbf{X}) + 2 \cdot \textproc{Par}(T|\mathbf{X}) \\
\textproc{BIC}(T|\mathbf{X}) & \equiv \textproc{Dev}(T|\mathbf{X}) + \log(n) 
\cdot \textproc{Par}(T|\mathbf{X})
\end{align*}
where $n$ is the number of samples.
Information criteria can be applied by using as the performance function \textproc{Perf} the information criterion value of a model, and a selection criterion \textproc{C} that simply compares the performance of two models, giving preference to the one with the lowest value.
Alternatively, one could check that the difference in scores is larger than some constant.

In case of nested models, selecting a model based on AIC or BIC directly corresponds to a likelihood-ratio test for some significance level $\alpha$.
We will show this for the BIC score next; the derivation for AIC is similar.
Let $\mathbf{X}$ and $\mathbf{X} \cup \mathbf{Y}$ be two candidate variables sets.
$\mathbf{X} \cup \mathbf{Y}$ is selected (that is, the null hypothesis is rejected) if $\textproc{BIC}(T|\mathbf{X} \cup \mathbf{Y}) < \textproc{BIC}(T|\mathbf{X})$, or equivalently if $\textproc{Dev}(T|\mathbf{X}) - \textproc{Dev}(T|\mathbf{X} \cup \mathbf{Y}) > \log(n) \cdot (\textproc{Par}(T|\mathbf{X} \cup \mathbf{Y}) - \textproc{Par}(T|\mathbf{X}))$.
Note that the left-hand side term equals the statistic of a likelihood-ratio test, whereas the right-hand size corresponds to the critical value.
The statistic follows a $\chi^2$ distribution with $k = \textproc{Par}(T|\mathbf{X} \cup \mathbf{Y}) - \textproc{Par}(T|\mathbf{X})$ degrees of freedom, and thus, the significance level equals $\alpha = 1 - F(\log(n) \cdot k; k)$, where $F(v;k)$ is the $\chi^2$ cdf with $k$ degrees of freedom at value $v$.

An issue with information criteria is that they are not designed for cases where the number of predictors $p$ is larger than the number of samples $n$ \citep{EBIC2008}, leading to a high false discovery rate.
An extension of BIC that deals with this problem, called extended Bayesian information criterion (EBIC), has been proposed by \cite{EBIC2008}.
EBIC is defined as $$\textproc{BIC}_\gamma (T|\mathbf{X}) = \textproc{BIC}(T|\mathbf{X}) + 2 \gamma \log \tau(\mathbf{X})$$ 
where $\gamma$ is a parameter taking values in $[0,1]$, and $\tau(\mathbf{X}) = \binom{p}{|\mathbf{X}|}$ where $|\mathbf{X}|$ is size of $\mathbf{X}$ and $p$ is the total number of predictors.
Note that, when $\gamma = 0$, then $\textproc{BIC}_\gamma (T|\mathbf{X}) = \textproc{BIC}(T|\mathbf{X})$.
The authors propose to use $\gamma = 1 - 1/(2\kappa)$, where $\kappa$ is obtained by solving $n = p^\kappa$ for $\kappa$, where $n$ is the number of samples (see Section 5 in \citep{EBIC2008}).

\subsection{Bayesian Networks and Maximal Ancestral Graphs}
We will briefly introduce Bayesian networks and maximal ancestral graphs, which we will use to show theoretical properties of the proposed algorithm.
For a comprehensive introduction to Bayesian networks and maximal ancestral graphs we refer the reader to \citep{Spirtes2000, SpirtesRichardson2002}.

Let $\mathbf{V}$ be a set of random variables. 
A \textbf{directed acyclic graph} (DAG) is a graph that only contains directed edges ($\rightarrow$) and has no directed cycles.
A \textbf{directed mixed graph} is a graph that, in addition to directed edges also contains bi-directed edges ($\leftrightarrow$).
The graphs contain no self-loops, and vertices can be connected only by a single edge.
Two vertices are called \textbf{adjacent} if they are connected by an edge.
An edge between $X$ and $Y$ is called \textbf{into} $Y$ if $X \rightarrow Y$ or $X \leftrightarrow Y$.
A \textbf{path} in a graph is a sequence of unique vertices $\langle V_1, \dots, V_k \rangle$ such that each consecutive pair of vertices is adjacent.
The first and last vertices in a path are called \textbf{endpoints}.
A path is called directed if $\forall 1 \leq i \leq k$, $V_i \rightarrow V_{i+1}$.
If $X \rightarrow Y$ is in a graph, then $X$ is a \textbf{parent} of $Y$ and $Y$ a \textbf{child} of $X$.
A vertex $W$ is a \textbf{spouse} of $X$, if both share a common child.
A vertex $X$ is an \textbf{ancestor} of $Y$, and $Y$ is a \textbf{descendant} of $X$, if $X = Y$ or there is a directed path from $X$ to $Y$.
A triplet $\langle X,Y,Z \rangle$ is called a \textbf{collider} if $Y$ is adjacent to $X$ and $Z$, and both, $X$ and $Z$ are into $Y$.
A triplet $\langle X,Y,Z \rangle$ is called \textbf{unshielded} if $Y$ is adjacent to $X$ and $Z$, but $X$ and $Z$ are not adjacent.
A path $p$ is called a \textbf{collider path} if every non-endpoint vertex is a collider on $p$.

\textbf{Bayesian networks} (BNs) consist of a DAG $\mathcal{G}$ and a probability distribution $\mathcal{P}$ over a set of variables $\mathbf{V}$.
The DAG represents dependency relations between variables in $\mathbf{V}$ and is linked with $\mathcal{P}$ through the \textbf{Markov condition}, which states that each variable is conditionally independent of its non-descendants given its parents.
Those are not the only independencies encoded in the DAG; the Markov condition entails additional independencies, which can be read from the DAG using a graphical criterion called \textbf{d-separation} \citep{Verma1988, Pearl1988}.
In order to present the d-separation criterion we first introduce the notion of blocked paths.
A (not necessarily directed) path $p$ between two nodes $X$ and $Y$ is called \textbf{blocked} by a set of nodes $\mathbf{Z}$ if there is a node $V$ on $p$ that is a collider and, neither $V$ nor any of its descendants are in $\mathbf{Z}$, or if $V$ is not a collider and it is in $\mathbf{Z}$.
If all paths between $X$ and $Y$ are blocked by $\mathbf{Z}$, then $X$ and $Y$ are \textbf{d-separated} given $\mathbf{Z}$; otherwise $X$ and $Y$ are \textbf{d-connected} given $\mathbf{Z}$.
The \textbf{faithfulness condition} states that all and only those conditional independencies in $\mathcal{P}$ are entailed by the Markov condition applied to $\mathcal{G}$.
In other words, the faithfulness condition requires that two variables $X$ and $Y$ are d-separated given a set of variables $\mathbf{Z}$ if and only if they are conditionally independent given $\mathbf{Z}$.

Bayesian networks are not closed under marginalization: a marginalized DAG, containing only a subset of the variables of the original DAG, may not be able to exactly represent the conditional independencies of the marginal distribution \citep{SpirtesRichardson2002}.
\textbf{Directed maximal ancestral graphs} (DMAGs) \citep{SpirtesRichardson2002} are an extension of BNs, which are able to represent such marginal distributions, that is, they admit the presence of latent confounders.
The graphical structure of a DMAG is a directed mixed graph with the following restrictions: (i) it contains no directed cycles, (ii) it contains no almost directed cycles, that is, if $X \leftrightarrow Y$ then neither $X$ nor $Y$ is an ancestor of the other, and (iii) there is no primitive inducing path between any two non-adjacent vertices, that is, there is no path $p$ such that each non-endpoint on $p$ is a collider and every collider is an ancestor of an endpoint vertex of $p$.
The d-separation criterion analogue for DMAGs is called the \textbf{m-separation criterion}, and follows the same definition.

A \textbf{Markov blanket} of a variable $T$ is a \textbf{minimal} set of variables $\mathbf{MB}(T)$ that renders $T$ conditionally independent of all remaining variables $\mathbf{V} \setminus \mathbf{MB}(T)$.
In case faithfulness holds, and the distribution can be represented by a BN or DMAG, then the Markov blanket is \textbf{unique}.
For a BN, the Markov blanket of $T$ consists of its parents, children and spouses.
For DMAGs it is slightly more complicated: the Markov blanket of $T$ consists of its parents, children and spouses, as well as its district (all vertices that are reachable by bi-directed edges), the districts of its children and the parents of all districts \citep{Richardson2003}.
An alternative definition is given next.
\begin{definition}
\label{def:mb}
The Markov blanket of $T$ in a BN or DMAG consists of all vertices adjacent to $T$, as well as all vertices that are reachable from $T$ through a collider path.
\end{definition}

A proof sketch follows.
Recall that a collider path of length $k-1$ is of the form $X_1 *\rightarrow X_2 \dots X_{k-1} \leftarrow* X_k$, where the path between $X_2$ and $X_{k-1}$ contains only bi-directed edges.
Given this, it is easy to see that Definition~\ref{def:mb} includes vertices directly adjacent to $T$, its spouses (collider path of length 2), and in the case of DMAGs, vertices $D$ in the district of $T$ ($T \leftrightarrow \dots \leftrightarrow D$), vertices $D$ in the district of any children $C$ of $T$ ($T \rightarrow C \leftrightarrow \dots \leftrightarrow D$), and all parents $P$ of any vertex $D$ in some district ($T *\rightarrow \dots \leftrightarrow D \leftarrow P$).
As the previous cases capture exactly all possibilities of nodes reachable from $T$ through a collider path, Definition~\ref{def:mb} does not include any additional variables that are not in the Markov blanket of $T$.

\section{Speeding-up Forward-Backward Selection}
\label{sec:impr_fs}

The standard \FBS{} has two main issues.
The first is that it is slow: at each forward iteration, all remaining variables are reconsidered to find the best next candidate.
If $k$ is the total number of selected variables and $p$ is the number of input variables, the number of model evaluations \FBS{} (or in our case, independence tests) performs is of the order of $O(k p)$.
Although relatively low-dimensional datasets are manageable, it can be very slow for modern datasets which often contain thousands of variables.
The second problem is that it suffers from multiple testing issues, resulting in overfitting and a high false discovery rate.
This happens because it reconsiders all remaining variables at each iteration; variables will often happen to seem important simply by chance, if they are given enough opportunities to be selected.
As a result, it will often select a significant number of false positive variables \citep{Flom2007}.
This behavior is further magnified in high-dimensional settings.
Next, we describe a simple modification of \FBS{}, improving its running time while reducing the problem of multiple testing.

\subsection{The Early Dropping Heuristic}
\label{sec:heuristic_remaining}

\begin{algorithm}[t!]
\caption{Forward-Backward Selection with Early Dropping (\fFBS{K})}
\label{alg:fbed}
\begin{algorithmic}[1]
	\Require Dataset $\mathcal{D}$, Target $T$, Maximum Number of Runs $K$
	\Ensure Selected Variables $\mathbf{S}$
	\State $\mathbf{S} \gets \emptyset$ \Comment{\textit{Set of selected variables}}
	\State $K_{cur} \gets 0$ \Comment{\textit{Initializing current number of runs to 0}}
	\State
	\State \Comment{\textit{Forward phase: iterate until (a) run limit reached, or (b) $\mathbf{S}$ does not change}}
	\While{$K_{cur} \leq K \wedge$ $\mathbf{S}$ changes} 
	\State $\mathbf{S} \gets \textproc{OneRun}(\mathcal{D}, T, \mathbf{S})$
	\State $K_{cur} \gets K_{cur} + 1$
	\EndWhile	
	\State
	\State \Comment{\textit{Perform backward selection and return result}}	
	\State {\bfseries return} $\mathbf{\textproc{BackwardSelection}(\mathcal{D}, T, \mathbf{S})}$
	\algrule
	\Function{OneRun}{$\mathcal{D}$, $T$, $\mathbf{S}$}
	\State $\mathbf{R} \gets \mathbf{V_{\mathcal{D}}} \setminus \mathbf{S}$ \Comment{\textit{Set of remaining candidate variables}}
	\State \Comment{\textit{Forward phase: iterate until $\mathbf{R}$ is empty}}
	\While{$|\mathbf{R}| > 0$} 
	\State \Comment{\textit{Identify best variable $V_{best}$ out of $\mathbf{R}$, according to \textproc{Perf}}}
	\State $V_{best}$ $\gets$ $\argmax\limits_{V \in \mathbf{R}}$ \textproc{Perf}$(\mathbf{S} \cup V)$
	\State \Comment{\textit{Select $V_{best}$ if it increases performance according to criterion \textproc{C}}}
	\If{\textproc{Perf}$(\mathbf{S} \cup V_{best}) \underset{\scriptscriptstyle{\textproc{C}}}{>} $ \textproc{Perf}$(\mathbf{S})$}
	\State $\mathbf{S}$ $\gets$ $\mathbf{S} \cup V_{best}$
	\EndIf
	\State \Comment{\textit{Drop all variables not satisfying \textproc{C}}}	
	\State $\mathbf{R}$ $\gets$ $\{V : V \in \mathbf{R} \wedge V \neq V_{best} \wedge \textproc{Perf}(\mathbf{S} \cup V) \underset{\scriptscriptstyle{\textproc{C}}}{>}  \textproc{Perf}(\mathbf{S})\}$
	\EndWhile
	\State {\bfseries return} $\mathbf{S}$
	\EndFunction
\end{algorithmic}
\end{algorithm}

We propose the following modification: after each forward iteration, remove all variables that do not satisfy the criterion $C$ for the current set of selected variables $\mathbf{S}$ from the remaining variables $\mathbf{R}$.
In our case, those variables are the ones that are conditionally independent of $T$ given $\mathbf{S}$.
The idea is to quickly reduce the number of candidate variables $\mathbf{R}$, while keeping many (possibly) relevant variables in it.
The forward phase terminates if no more variables can be selected, either because there is no informative variable or because $\mathbf{R}$ is empty; to distinguish between forward and backward phases, we will call a forward phase with early dropping a \textbf{run}.
Extra runs can be performed to reconsider variables dropped previously.
This is done by retaining the previously selected variables $\mathbf{S}$ and initializing the set of remaining variables to all variables which have not been selected yet, that is $\mathbf{R} = \mathbf{V}_{\mathcal{D}} \setminus \mathbf{S}$.
The backward phase employed afterwards is identical to the standard backward-selection algorithm (see Algorithm~\ref{alg:fbs}).
Depending on the number of additional runs $K$, this defines a family of algorithms, which we call \textbf{Forward Backward Selection with Early Dropping} (\fFBS{K}), shown in Algorithm~\ref{alg:fbed}.
The function \textproc{OneRun} shown in the bottom of Algorithm~\ref{alg:fbed}, performs one run until no variables remain in $\mathbf{R}$.
Three interesting members of this family are the \fFBS{0}, \fFBS{1} and \fFBS{\infty} algorithms.
\fFBS{0} performs the first run until termination, \fFBS{1} performs one additional run and \fFBS{\infty} performs runs until no more variables can be selected.
We will focus on those three algorithms hereafter.

The heuristic is inspired by the theory of Bayesian networks and maximal ancestral graphs \citep{Spirtes2000, SpirtesRichardson2002}. 
Similar heuristics have been applied by Markov blanket based algorithms such as MMPC \citep{Tsamardinos2003MMPC} and HITON-PC \citep{Aliferis2003HITON} successfully in practice and in extensive comparative evaluations \citep{Aliferis2010JMLR}.
These algorithms also remove variables from consideration, and specifically the ones that are conditionally independent given some \textbf{subset} of the selected variables. 
The connections of \fFBS{K} to graphical models and Markov blankets are presented in Section~\ref{sec:mbfbs}.

\subsection{Comparing \fFBS{K} and \FBS{}}
Next, we will compare \fFBS{0}, \fFBS{1} and \fFBS{\infty} to \FBS{}, and will show their connections to Bayesian networks and maximal ancestral graphs.
An extensive evaluation on real data will be presented in Section~\ref{sec:exp:ffbs}.
We proceed with some comments on computational cost and multiple testing of \fFBS{K} relative to \FBS{}.

\subsubsection{Computational Speed}
It is relatively easy to see that \fFBS{0} is faster than \FBS{}, as it quickly excludes many variables.
Usually, the same also holds for \fFBS{1} and even for \fFBS{\infty}.
To see how the latter can be the case, consider the following example.
For the sake of argument assume that both \FBS{} and \fFBS{\infty} will select the same $k$ variables and in the same order.
Then, one can see that \fFBS{\infty} will make as most as many tests as take as \FBS{} (up to a $O(p)$ factor): \FBS{} will consider $p + (p-1) + \dots + (p-k+1)$ variables in total, whereas \fFBS{\infty} will consider at most that many (plus $p-k+1$ in the last iteration, where all remaining variables are reconsidered once) in one of the following cases: (a) no variable is dropped in any iteration, or (b) all variables are dropped in each iteration, leading to a new run where all remaining variables have to be reconsidered.
As neither of those cases is typical, \fFBS{\infty} will usually be faster than \FBS{}.
The actual speed-up can not be quantified, as it highly depends on the input data.

\subsubsection{Multiple Testing}
\begin{table*}[!t]
\centering
	\caption{
	Average number of selected variables by \FBS{} and \fFBS{K} over 100 randomly generated datasets and outcomes with 200 samples, for varying number of predictors $p$ and significance levels $\alpha$.
	The rows on the bottom show the variables selected relative to $\alpha \cdot p$, the expected number of type I errors.
	Overall, \FBS{} and \fFBS{\infty} have high type I error, increasing with variable size and significance level, while \fFBS{0} and \fFBS{1} control type I error rate, improving with variable size and significance level.
		}
    \label{tbl:mt}
  \fontsize{8pt}{10pt}\selectfont
  \begin{tabular}{llrrrrrr}
  \cmidrule(r){2-8}
  	& $p$ & \multicolumn{3}{c}{100} & \multicolumn{3}{c}{200} \\ \cmidrule(r){3-5} \cmidrule(r){6-8}
  	& $\alpha$ & 0.01 & 0.05 & 0.1 & 0.01 & 0.05 & 0.1\\ \midrule
\multirow{5}{*}{\rotatebox{0}{\#vars}} & \fFBS{0} & 0.9 & 3.3 & 6.3 & 1.8 & 5.6 & 9.3 \\
& \fFBS{1} & 1.1 & 4.6 & 9.3 & 2.5 & 8.7 & 15.9 \\
& \fFBS{\infty} & 1.2 & 5.9 & 14.4 & 2.9 & 22.3 & 39.3 \\
& \FBS{} & 1.2 & 5.8 & 14.2 & 2.8 & 20.8 & 38.1 \\
& $\alpha \cdot p$ & 1.0 & 5.0 & 10.0 & 2.0 & 10.0 & 20.0 \\ \midrule
\multirow{4}{*}{\rotatebox{0}{$\frac{\#vars}{\alpha \cdot p}$}} & \fFBS{0} & 92.0\% & 65.0\% & 62.7\% & 87.5\% & 56.0\% & 46.6\% \\
& \fFBS{1} & 113.0\% & 91.2\% & 92.5\% & 122.5\% & 86.8\% & 79.4\% \\
& \fFBS{\infty} & 124.0\% & 117.6\% & 143.8\% & 142.5\% & 223.1\% & 196.3\% \\
& \FBS{} & 122.0\% & 115.2\% & 142.4\% & 141.0\% & 207.9\% & 190.6\% \\
    \bottomrule
  \end{tabular}
\end{table*}
	
The idea of early dropping of variables used by \fFBS{K} does not only reduce the running time, but also reduces the problem of multiple testing, in some sense.
Specifically, it reduces the number of variables falsely selected due to type I errors.
In general, the number of type I errors is directly related to the total number of variables considered in all forward iterations.
Thus, the effect highly depends on the value of K used by \fFBS{K}, with higher values of K leading to more false selections.
We will demonstrate this for \fFBS{0} by considering a simple scenario, where none of the candidate variables are predictive for the outcome.
Then, in the worst case, \fFBS{0} will select about $\alpha \cdot p$ of the variables on average (where $\alpha$ is the significance level), since all other variables will be dropped in the first iteration.
This stems from the fact that, under the null hypothesis of conditional independence, the p-values are uniformly distributed.
In practice, the number of selected variables will be even lower, as \fFBS{0} will keep dropping variables after each variable inclusion.
On the other hand, \FBS{} may select a much larger number of variables, since each variable is given the chance to be included in the output at each iteration and will often do so, simply by chance.

We will not study the problem of multiple testing in depth, and only performed a small simulation to investigate the behavior of \fFBS{0}, \fFBS{1}, \fFBS{\infty} and how they compare to \FBS{}.
We generated 100 normally distributed datasets with 200 samples each, a uniformly distributed random binary outcome, and considered two different variable sizes, $p = 100$ and $p = 200$.
All variables are generated randomly, and there is no dependency between any of them.
Thus, a false positive rate of about $\alpha$ is expected, if no adjustment is done to control the false discovery rate.
We then ran all algorithms using a logistic regression based independence test for three values of $\alpha$, 0.01, 0.05 and 0.1.
The results are summarized in Table~\ref{tbl:mt}.
We can see how the value of K used by \fFBS{K} affects the number of falsely selected variables.
\FBS{} and \fFBS{\infty} perform similarly, with \fFBS{\infty} selecting slightly more variables.
On the other hand, \fFBS{0} and \fFBS{1} always select fewer variables than \FBS{} and \fFBS{\infty}, with the number of variables falling below $\alpha \cdot p$ for both, in contrast to \FBS{} and \fFBS{\infty} which select more than $\alpha \cdot p$ variables.
Furthermore, what is more interesting is that \FBS{} and \fFBS{\infty} tend to perform worse with increasing $\alpha$ and $p$, whereas the opposite effect can be observed for \fFBS{0} and \fFBS{1}.

\subsubsection{Theoretical Properties}
\label{sec:fs:theory}
Due to early dropping of variables, the distributions under which \fFBS{K} and \FBS{} perform optimally are not the same.
For all versions of \fFBS{K} except for \fFBS{\infty} it is relatively straightforward to construct examples where \FBS{} is able to identify variables that can not be identified by \fFBS{K}.
We give an example for \fFBS{0}.
\fFBS{0} may remove variables that seem uninformative at first, but become relevant if considered in conjunction with other variables.
For example, let $X = T + Y$, $T \sim \mathcal{N}(\mu_T, \sigma^2_T)$ and $Y \sim \mathcal{N}(\mu_Y, \sigma^2_Y)$ where $T$ is the outcome and $X, Y$ are two predictors.
Then, $X$ will be found relevant for predicting $T$ but $Y$ will be discarded, as it does not give any information about $T$ by itself.
However, after selecting $X$, $Y$ becomes relevant again, but \fFBS{0} will not select it as it was dropped in the first iteration.
Surprisingly, in practice this does not seem to significantly affect the quality of \fFBS{0}.
In contrast, \fFBS{0} often gives better results, while also selecting fewer variables than \FBS{} (see Section~\ref{sec:exp:fs:opt}).

As mentioned above, it is not clear how \FBS{} and \fFBS{\infty} are related in the general case; the special case in which distributions can be represented by Bayesian networks or maximal ancestral graphs is considered in Section~\ref{sec:mbfbs}.
For the general case we show that, although they do not necessarily give the same results, both identify what we call a minimal set of variables.

\begin{definition}[Minimal Variable Set]
Let $\mathbf{V_{\mathcal{D}}}$ be the set of all variables and $\mathbf{V_{sel}}$ a set of selected variables.
We call a set of variables $\mathbf{V_{sel}}$ \textbf{minimal} with respect to some outcome $T$, if:
\begin{enumerate}
\item No variable can be removed from $\mathbf{V_{sel}}$ given the rest of the selected variables, that is, $\forall V_i \in \mathbf{V_{sel}}, \conddep{T}{V_i}{\mathbf{V_{sel}}\setminus V_i}$ holds.
\item Let $\mathbf{V_{rem}} = \mathbf{V_{\mathcal{D}}} \setminus \mathbf{V_{sel}}$. No variable from $\mathbf{V_{rem}}$ can be included in $\mathbf{V_{sel}}$, that is, $\forall V_i \in \mathbf{V_{rem}}, \condind{T}{V_i}{\mathbf{V_{sel}}}$ holds.
\end{enumerate}
\end{definition}

In words, a minimal set is a set such that no single variable can be included to or removed from using forward and backward iterations respectively, or, in other words, is a local optimum for stepwise algorithms.
Note that, although no single variable is informative for $T$ if looked at separately, there may be sets of variables that are informative if considered jointly.
A simple example is if all variables are binary and $T = X \oplus Y$, where $\oplus$ is the logical XOR operator.
In this case $\mathbf{V_{sel}} = \emptyset$ is minimal, as neither $X$ nor $Y$ are dependent with $T$, even though the set $\{X,Y\}$ fully determines $T$.
Forward selection based algorithms are usually not able to identify such relations.
Next, we show that both algorithms identify minimal variable sets.

\begin{theorem}\label{thm:fbs}
Any set of variables $\mathbf{V_{sel}}$ selected by \FBS{} is minimal.
\end{theorem}

\begin{proof}
See Appendix~\ref{app:theory}.
\end{proof}

\begin{theorem}\label{thm:ffbs}
Any set of variables $\mathbf{V_{sel}}$ selected by \fFBS{\infty} is minimal.
\end{theorem}

\begin{proof}
See Appendix~\ref{app:theory}.
\end{proof}

It is important to note that, although both \fFBS{\infty} and \FBS{} will select minimal variable sets, it is not guaranteed that they will select the same set of variables.

\subsubsection{Identifying Markov Blankets with \fFBS{K}}
\label{sec:mbfbs}

We proceed by showing that \fFBS{1} and \fFBS{\infty} identify the Markov blanket of a BN and DMAG respectively, assuming (a) that the distribution can be faithfully represented by the respective graph, and (b) that the algorithms have access to an \textbf{independence oracle}, which correctly determines  whether a given conditional (in)dependence holds.
This also holds for \FBS{} but will not be shown here; proofs for similar algorithms exist \citep{Margaritis2000, Tsamardinos2003IAMB} and can be easily adapted to \FBS{}.
For \fFBS{0} it can be shown that it selects a superset of the variables that are adjacent to $T$ in the graph; this can be shown using the fact that adjacent variables are dependent with $T$ given any subset of variables by the properties of d-separation / m-separation.

\begin{theorem}\label{thm:ffbs1mb}
If the distribution can be faithfully represented by a Bayesian network, then \fFBS{1} identifies the Markov blanket of the target $T$.
\end{theorem}

\begin{proof}
See Appendix~\ref{app:theory}.
\end{proof}

\begin{theorem}\label{thm:ffbsinfmb}
If the distribution can be faithfully represented by a directed maximal ancestral graph, then \fFBS{\infty} identifies the Markov blanket of the target $T$.
\end{theorem}

\begin{proof}
See Appendix~\ref{app:theory}.
\end{proof}

\section{Experimental Evaluation of \fFBS{K}}
\label{sec:exp:ffbs}

\setlength{\tabcolsep}{.5em}
\begin{table*}[!t]
\centering
	\caption{Binary classification datasets used in the experimental evaluation.}
    \label{tbl:data}
   \fontsize{7pt}{9pt}\selectfont
  \begin{tabular}{lccclll}
    \toprule
    Dataset &
    n &
    p &
    P(T = 1) &
    Type &
    Domain &
    Source
\\
     \midrule
     \datasetref{musk (v2)}{6598}{166}{0.15}{Real}{Musk Activity Prediction}{UCI ML Repository}{\citep{Dietterich1994}}
     \midrule
     \datasetref{sylva}{14394}{216}{0.94}{Mixed}{Forest Cover Types}{WCCI 2006 Challenge}{\citep{Guyon2006}}
     \midrule
     \datasetref{madelon}{2600}{500}{0.5}{Integer}{Artificial}{NIPS 2003 Challenge}{\citep{Guyon2004}}
     \midrule
     \datasetref{secom}{1567}{590}{0.93}{Real}{Semi-Conductor Manufacturing}{UCI ML Repository}{M. McCann, A. Johnston}
     \midrule
     \datasetref{gina}{3568}{970}{0.51}{Real}{Handwritten Digit Recognition}{WCCI 2006 Challenge}{\citep{Guyon2006}}
     \midrule
     \datasetref{hiva}{4229}{1617}{0.96}{Binary}{Drug discovery}{WCCI 2006 Challenge}{\citep{Guyon2006}}
     \midrule
     \datasetref{gisette}{7000}{5000}{0.5}{Integer}{Handwritten Digit Recognition}{NIPS 2003 Challenge}{\citep{Guyon2004}}
     \midrule
     \datasetref{p53 Mutants}{16772}{5408}{0.01}{Real}{Protein Transcriptional Activity}{UCI ML Repository}{\citep{Danziger2006}}
     \midrule
     \datasetref{arcene}{200}{10000}{0.56}{Binary}{Mass Spectrometry}{NIPS 2003 Challenge}{\citep{Guyon2004}}
     \midrule
     \datasetref{nova}{1929}{16969}{0.72}{Binary}{Text classification}{WCCI 2006 Challenge}{\citep{Guyon2006}}
     \midrule
     \datasetref{dexter}{600}{20000}{0.5}{Integer}{Text classification}{NIPS 2003 Challenge}{\citep{Guyon2004}}
     \midrule	
     \datasetref{dorothea}{1150}{100000}{0.9}{Binary}{Drug discovery}{NIPS 2003 Challenge}{\citep{Guyon2004}}
    \bottomrule
  \end{tabular}
\end{table*}

In this section we evaluate three versions of \fFBS{K}, namely \fFBS{0}, \fFBS{1}, \fFBS{\infty}, and compare them to the standard \FBS{} algorithm as well as feature selection with LASSO \citep{Tibshirani1996}.
We used 12 binary classification datasets, with sample sizes ranging from 200 to 16772 and number of variables between 166 and 100000.
The datasets were selected from various competitions and the UCI repository \citep{Dietterich1994}, and were selected to cover a wide range of variable and sample sizes.
A summary of the datasets is shown in Table~\ref{tbl:data}.
All experiments were performed in MATLAB, running on a desktop computer with an Intel i7-4790K processor and 32GB of RAM.

We proceed by describing in detail the experimental setup, that is, all algorithms used, their hyper-parameters, and how we performed model selection and performance estimation.
Then, we compare \fFBS{K} to \FBS{}, in terms of predictive ability, number of selected variables and number of performed independence tests.
For the sake of simplicity, we will hereafter refer to comparisons of BIC/EBIC scores as independence tests (see Section~\ref{sec:ic} for their relation).
Afterwards, we compare \fFBS{K} and \FBS{} to feature selection with LASSO.
We performed two experiments: (a) one where we optimize the regularization parameter $\lambda$ for LASSO, allowing LASSO to select any number of variables, and (b) one where we restrict LASSO to select a fixed number of variables.
In the latter, we restrict LASSO to select as many variables as either \fFBS{K} or \FBS{}.
This is done for two reasons: (a) because LASSO tends to select many variables otherwise, giving it an advantage over \FBS{} and \fFBS{K} in terms of predictive performance, and (b) because this allows us to evaluate how well \fFBS{K} and \FBS{} order the variables in comparison to LASSO.
Appendix~\ref{app:results} contains all results in detail, as well as results of the running time of each method.

\subsection{Experimental Setup}

\textbf{Algorithms}.
We evaluated three instances of the proposed \fFBS{K} algorithm, \fFBS{0}, \fFBS{1} and \fFBS{\infty}.
\fFBS{K} was compared to the standard \FBS{} algorithm and to feature selection with LASSO (a.k.a. L1-regularized logistic regression), called LASSO-FS hereafter.
We used the glmnet implementation \citep{glmnet} of LASSO, with all parameters set to their default values except for the maximum number of $\lambda$ values.

\noindent\textbf{Performance metrics}.
For model selection, we used the deviance as a performance metric.
In the experiments we compute two metrics, the area under the ROC curve (AUC) and classification accuracy (ACC); we do not report the deviance, as it isn't as interpretable as the AUC and accuracy metrics.
To get a binary prediction from logistic regression models, we used the threshold 0.5 on the predicted probabilities.

\noindent\textbf{Feature selection hyper-parameters}.
As selection criteria we used the EBIC criterion \citep{EBIC2008} and a nested likelihood-ratio independence test based on logistic regression.
The $\gamma$ parameter for the EBIC criterion was set to $0, 0.5, 1$ and the default value $1 - 0.5 \cdot \log(n) / \log(p)$, all of which are special values for $\gamma$ and were used in \citep{EBIC2008}.
For the independence test (IT), we used $0.001, 0.01, 0.05$ and $0.1$ as values for the significance level $\alpha$, covering a range of commonly used values.
For LASSO-FS we set the maximum number of values for $\lambda$ to 1000.

\noindent\textbf{Model selection and performance estimation protocol}.
In order to perform model selection and performance estimation, we used a $60/20/20$ stratified split of the data using $60\%$ as a training set, $20\%$ as a validation set and the remaining $20\%$ as a test set.
A hyper-parameter configuration (called configuration hereafter) is defined as a combination of a feature selection algorithm and its hyper-parameters, as well as a modeling algorithm and its hyper-parameters.
Given a set of configurations, the best one is chosen by training models for all of them on the training set and selecting the configuration of the model that performs best on the validation set.
Finally, the predictive performance of that configuration is obtained by training a final model on the pooled training and validation sets, and evaluating it on the test set.
To get more stable performance estimates, this procedure was repeated multiple times for different splits and averages over repetitions are reported.
For datasets with more than 1000 samples, the number of repetitions was set to 10, and to 50 for the rest.
Next, we describe the model selection procedure in more detail.

\subsection{Model Selection Procedure}
In order to obtain a predictive model after performing feature selection, we ran an L1-regularized logistic regression (called LASSO-PM hereafter); a justification of that choice is given below.
Again, we used a total of 1000 values for $\lambda$.

For LASSO-FS (without a limit on the number of variables), a model was obtained by running LASSO on the training set, evaluating the out-of-sample deviance of each produced model on the validation set, and selecting the $\lambda$ value that corresponds to the best performing model.
This can be done, as LASSO is used both for feature selection and modeling, and thus it is not necessary to first select a set of features, and then train additional models to find the best configuration.
In contrast, for each configuration involving \FBS{} or \fFBS{K}, we first perform feature selection and then train multiple models with LASSO-PM using the selected features on the training set, and then select the best $\lambda$ value using the validation set.
When limiting the number of variables LASSO-FS can select to $M$, we first run LASSO-FS with 10000 values for $\lambda$ and select the variables of the first model containing at least $M$ non-zero coefficients.
As with \FBS{} and \fFBS{K}, we then train multiple models with LASSO-PM and select the best $\lambda$ value on the validation set.

Note that, a set of configurations may contain one or multiple feature selection algorithms, as well as one or multiple hyper-parameters for each such algorithm.
In some experiments, we optimize each feature selection algorithm and hyper-parameter separately (for example, \FBS{} with IT and $\alpha = 0.01$), while in others we report the performance of a feature selection algorithm by optimizing over all hyper-parameters (for example, \FBS{} with IT and multiple $\alpha$ values).

\subsection*{Choice of LASSO-PM}
\label{sec:lassopm}

\noindent\textbf{Necessity of regularization}.
A natural choice for a modeling algorithm would be standard, unpenalized logistic regression.
One issue with this is that the resulting model would have inflated coefficients due to the feature selection procedure \citep{Flom2007}, reducing its predictive ability.
LASSO isn't affected by this as much, as it shrinks the coefficients while simultaneously performing feature selection, which is one of the reasons it is so successful.
Furthermore, for a fair comparison, one would also have to fit an unpenalized model for LASSO-FS, which wouldn't make much sense.
Therefore, it makes sense to \textit{use modeling algorithms that perform some kind of regularization}.
In order to support this choice, we repeated all experiments using unpenalized logistic regression as the modeling algorithm; the results are summarized in Appendix~\ref{app:results}.
In short, the results agree with what was expected, while also suggesting that methods selecting many features are mostly affected by this effect.

\noindent\textbf{Linear vs non-linear models}.
Some candidate methods with regularization are L1, L2 or elastic net regularized logistic regression, all of which are linear models, or non-linear models such as support vector machines or random forests, which (explicitly or implicitly) also perform some kind of regularization.
For a fair comparison between all feature selection methods, we argue that one should use a linear model instead of a non-linear one, as all feature selection methods used are based on logistic regression and thus will be able to identify linear (or monotonic) relations.
Running a non-linear model afterwards may favor methods which happen to select features that have a non-linear dependency with the outcome.
Furthermore, it is harder to tune a model such as SVMs, even with a linear kernel, than to tune a logistic regression model.
Between the three linear options, we chose L1-regularization (LASSO-PM) to avoid selecting features with LASSO-FS and then using a different model afterwards.
Note that, \textit{this choice may favor LASSO-FS, as it performs feature selection and modeling simultaneously}.

\noindent\textbf{Other considerations}.
As a final comment, we point out that the choice of LASSO-PM as a modeling algorithm tends to favor feature selection methods that select many features, as it will implicitly perform an additional feature selection step afterwards.
Thus, if for example an algorithm selects all important predictors but one, while another selects all important ones along with many irrelevant ones, the latter will in general perform better with LASSO-PM (depending on factors such as sample size and number of irrelevant predictors).
However, there is no reason to not perform regularization afterwards: apart from the reasons mentioned previously, in a real-world scenario one is not restricted to the model used while performing feature selection, but can (and should) try out other models.
Using LASSO-PM is a good compromise, as it performs shrinkage while being linear.
Thus, \textit{the modeling may favor algorithms that select many features, even if they are irrelevant or redundant}.
Of course, what really matters in practice is not only the performance achieved by the methods, but also the total number of selected features; methods that select many irrelevant features may perform similarly to methods selecting mostly relevant ones, but the latter are arguably more useful in practice.

\subsection{\FBS{} vs \fFBS{K}}

In this section we compare \fFBS{K} for $K = 0,1$ and $\infty$ with the standard \FBS{} algorithm, in terms of predictive performance, number of selected variables and number of performed independence tests.
\textit{The main goal of this comparison is to show that \fFBS{K} and \FBS{} perform similarly for the same hyper-parameters, with the former being faster}.

Model selection was performed for each algorithm and each hyper-parameter value separately.
We use the results of \FBS{} as a baseline, and compare them to \fFBS{K} when using the same hyper-parameters (for example, \FBS{} vs \fFBS{0} with IT and $\alpha$ = 0.01).
Results for when hyper-parameters are optimized are presented in the next section where we also compare all methods to LASSO-FS.
Next, we will present a summary of all results; tables containing detailed results for all algorithms can be found in Appendix~\ref{app:results}.

\begin{figure}[t!]
\begin{subfigure}[t]{0.475\textwidth}
\centering
\includegraphics[width=\textwidth]{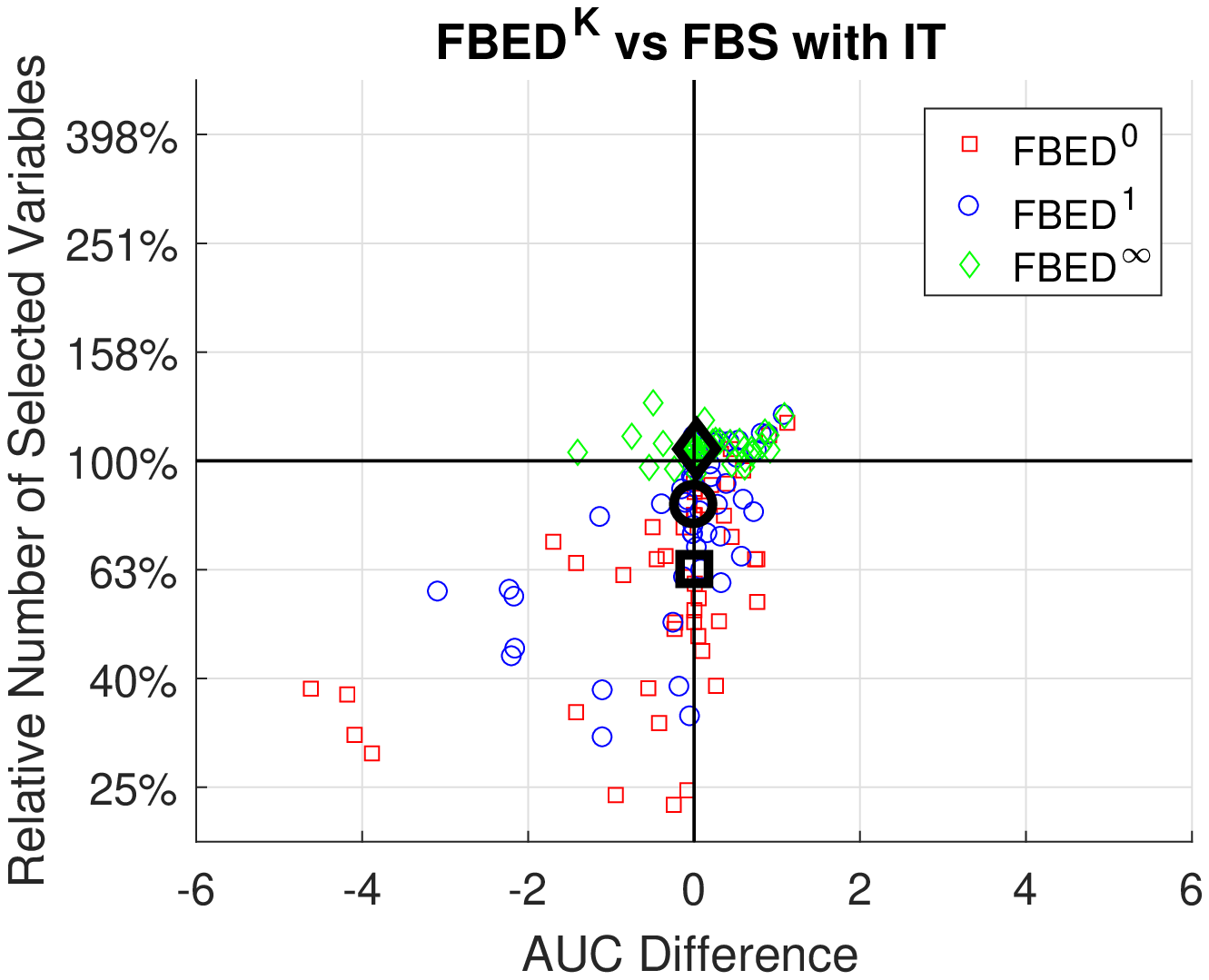}
\end{subfigure}
~
\begin{subfigure}[t]{0.475\textwidth}
\centering
\includegraphics[width=\textwidth]{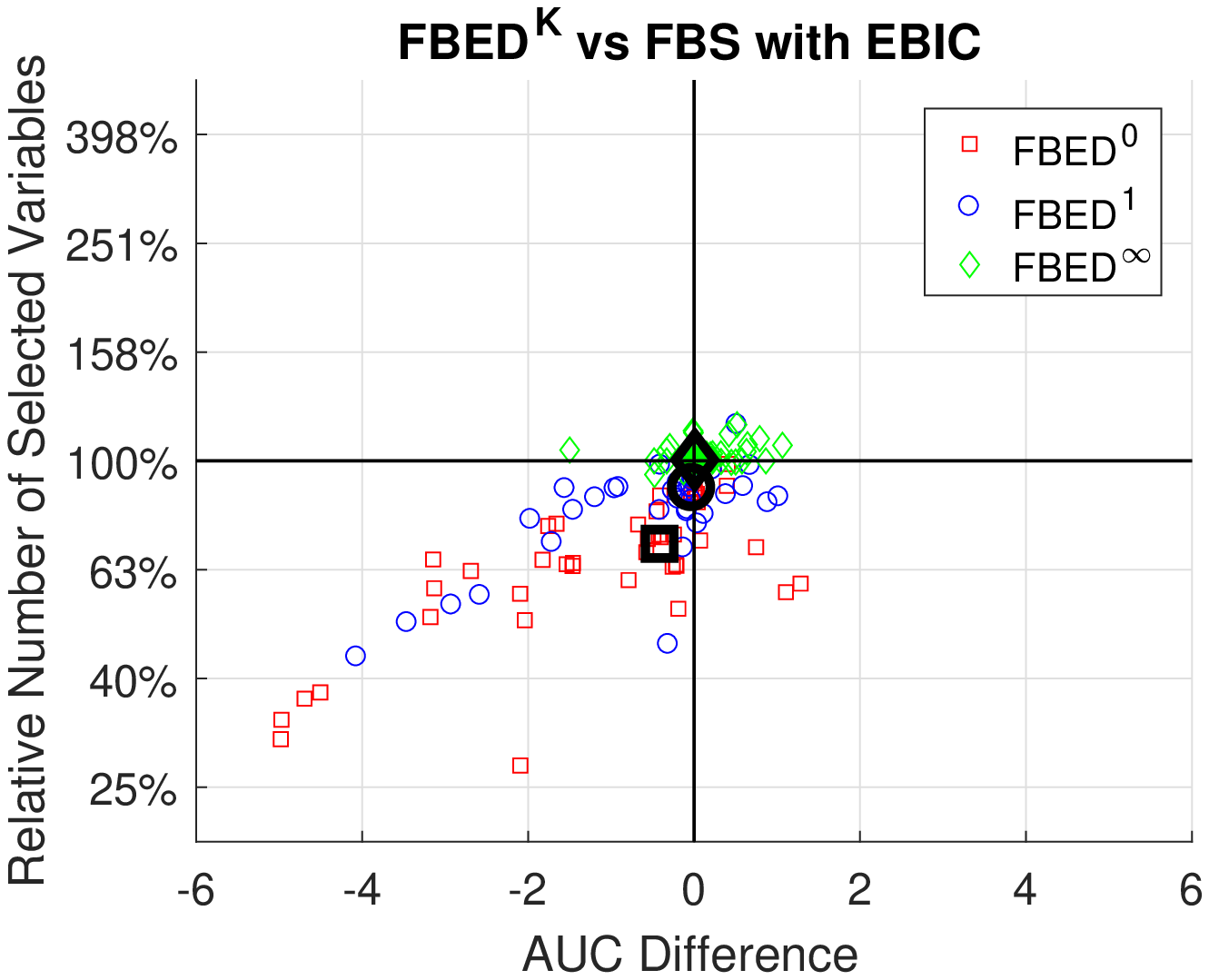}
\end{subfigure}
\\
\\
\begin{subfigure}[t]{0.475\textwidth}
\centering
\includegraphics[width=\textwidth]{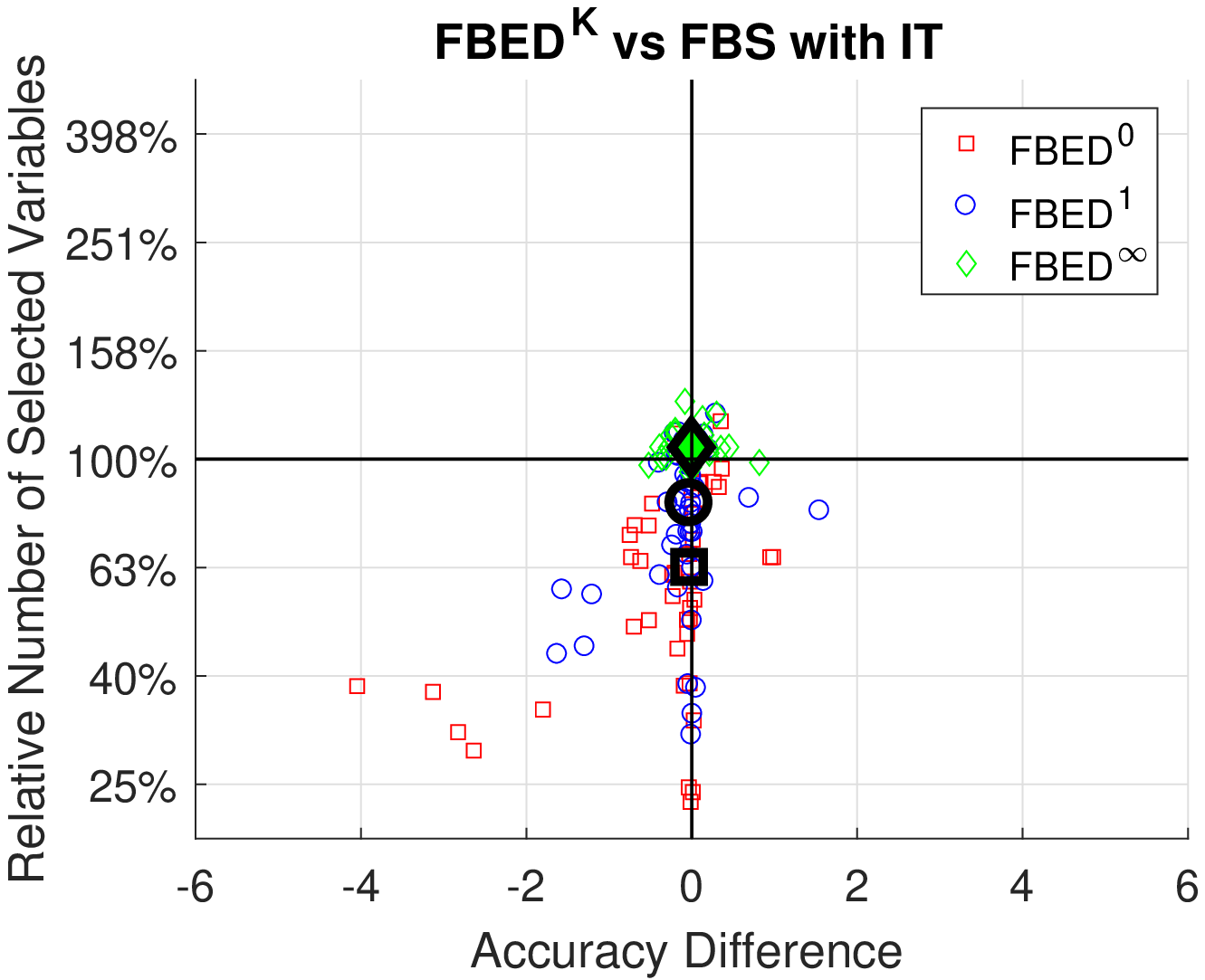}
\end{subfigure}
~
\begin{subfigure}[t]{0.475\textwidth}
\centering
\includegraphics[width=\textwidth]{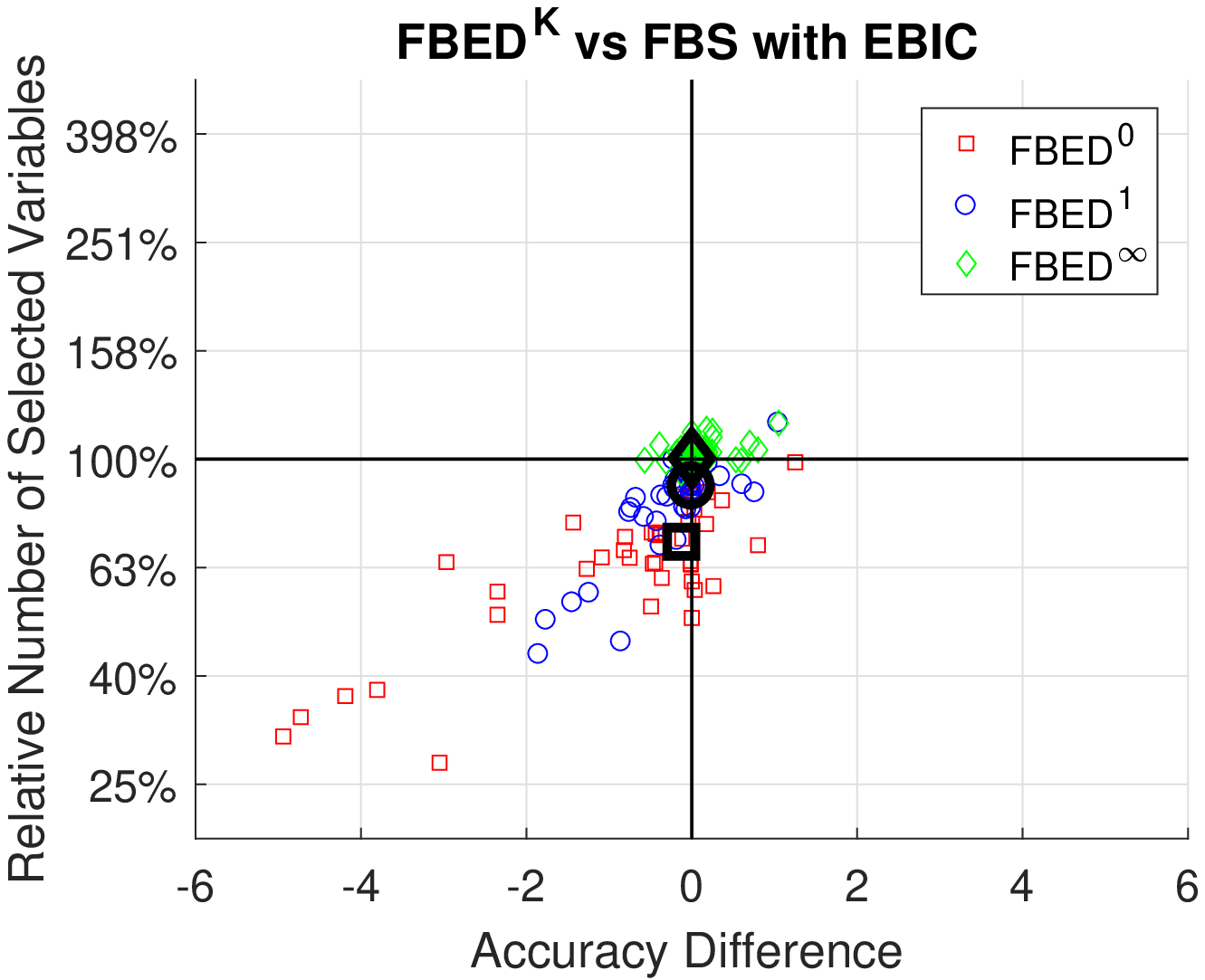}
\end{subfigure}
\caption{
\textbf{Predictive Performance vs Number of Selected Variables:} 
The x-axis shows the difference in performance of \FBS{} and \fFBS{K}, with positive values indicating that \fFBS{K} performs better than \FBS.
The y-axis shows the relative number of selected variables between \FBS{} and \fFBS{K}, with values below 100\% indicating that \fFBS{K} selects fewer variables.
Black points show the median value of the respective algorithm, computed separately for each axis.
In summary, \fFBS{0} and \fFBS{1} perform similarly to \FBS{} while selecting fewer variables, whereas \fFBS{\infty} is comparable to \FBS{}.
}
\label{fig:fsvsffs_perf}
\end{figure}

Figure~\ref{fig:fsvsffs_perf} shows how the algorithms compare in terms of predictive performance and number of selected variables.
The black points correspond to the median performance and relative number of selected variables of each algorithm, computed separately for each value, which means that the actual point does not necessarily correspond to any actual hyper-parameter value.
The median was chosen instead of the mean, because (a) it is more reasonable than averaging differences in performance across different datasets, and (b) as the mean would be misleading due to the existence of a few outliers which, as we explain below, correspond to a single dataset.
On the x-axis, the difference in AUC or ACC between \FBS{} and \fFBS{K} is shown (Performance(\FBS{}) - Performance(\fFBS{K})), while the y-axis shows the relative number of selected variables (SelectedVars(\fFBS{K})/SelectedVars(\FBS{})) on a logarithmic scale (that is, points that are equidistant from the horizontal line at 100\% are inversely related).
Thus, points in the upper-left corner are points in which \FBS{} outperforms \fFBS{K} both in terms of performance and selected variables, while the opposite holds for points in the lower-right corner.
Overall, it can be seen that most points fall close to the vertical line, that is, the algorithms perform similarly on most datasets and for most hyper-parameter values.
A more detailed comparison follows.

\begin{itemize}
\item \textbf{\FBS{} vs \fFBS{\infty}}: 
As expected from theory, they perform similarly, both in terms of predictive performance as well as in terms of selected variables. 
\fFBS{\infty} tends to select slightly more variables, as most points fall above the horizontal line.
\item \textbf{\FBS{} vs \fFBS{0} and \fFBS{1}}: 
Both versions compare favorably to \FBS{} in terms of number of selected variables, while usually displaying similar predictive performance.
\FBS{} performs better for some cases, which correspond to the musk dataset, but selects significantly more variables.
With the EBIC criterion \fFBS{0} and \fFBS{1} tend to perform slightly worse, which is due to the fact that EBIC tends to be overly conservative in some cases, especially with large values of $\gamma$.
\end{itemize}

Next, we compare the number of independence tests performed by each algorithm.
Again, we use \FBS{} as a baseline and compare them on the same hyper-parameter values.
Figure~\ref{fig:fsvsffs_tests} shows the distribution of the relative number of independence tests performed by each \fFBS{K} algorithm.
The blue line shows the median values.
As can be clearly seen, \fFBS{K} always outperforms \FBS{}.
Note that this comparison does not show the actual running time of each algorithm.
The running time of an independence test depends on the total number of variables involved in it, thus, methods selecting more variables will (most likely) take more time.
Therefore, in practice the difference in running time between \FBS{} and \fFBS{K} (especially \fFBS{0} and \fFBS{1}) may be even larger. 
Detailed results of the running time of each algorithm on the full datasets are shown in Appendix~\ref{app:results}.
Overall, \fFBS{0} and \fFBS{1} perform around 1-2 orders if magnitude fewer tests than \FBS{}, with \fFBS{1} being slightly slower, while \fFBS{\infty} performs typically around 25\%-30\% of the tests.

In general, \fFBS{0} and \fFBS{1} are preferable over \FBS{} and \fFBS{\infty}.
For problems with many variables, \fFBS{0} and \fFBS{1} clearly dominate \FBS{} and \fFBS{\infty} in terms of running time.
Furthermore, \fFBS{0} and \fFBS{1} may be preferable as they select fewer variables, which is especially important for small sample sizes, where selecting many variables may lead to loss of power and overfitting.
In large sample size settings, and if the number of variables is relatively small (in the hundreds or up to few thousands), \fFBS{\infty} and \FBS{} are reasonable choices, with the former being more preferable, as it is able to scale to larger variable sizes.

\begin{figure}[t!]
\begin{subfigure}[t]{0.475\textwidth}
\centering
\includegraphics[width=\textwidth]{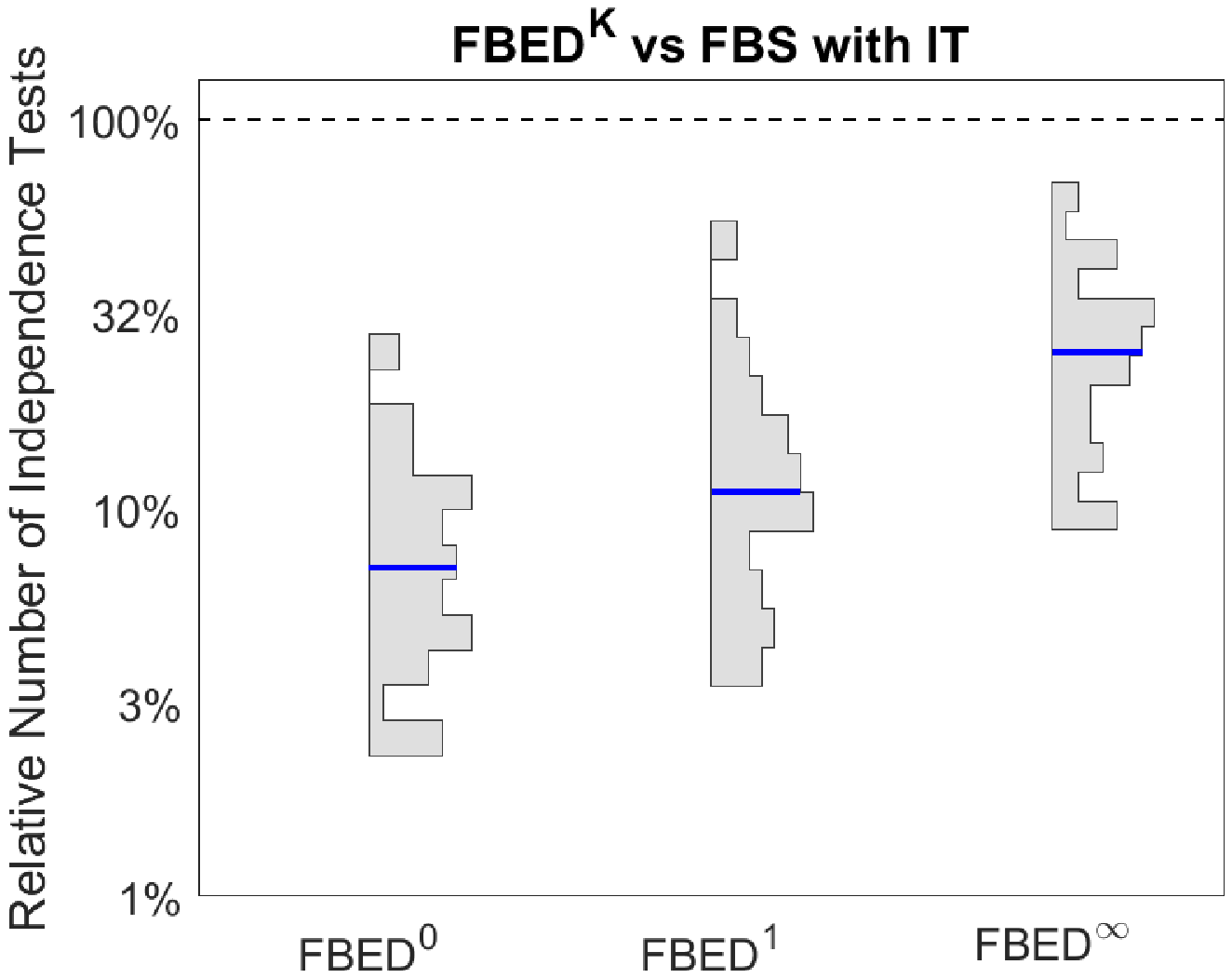}
\end{subfigure}
~
\begin{subfigure}[t]{0.475\textwidth}
\centering
\includegraphics[width=\textwidth]{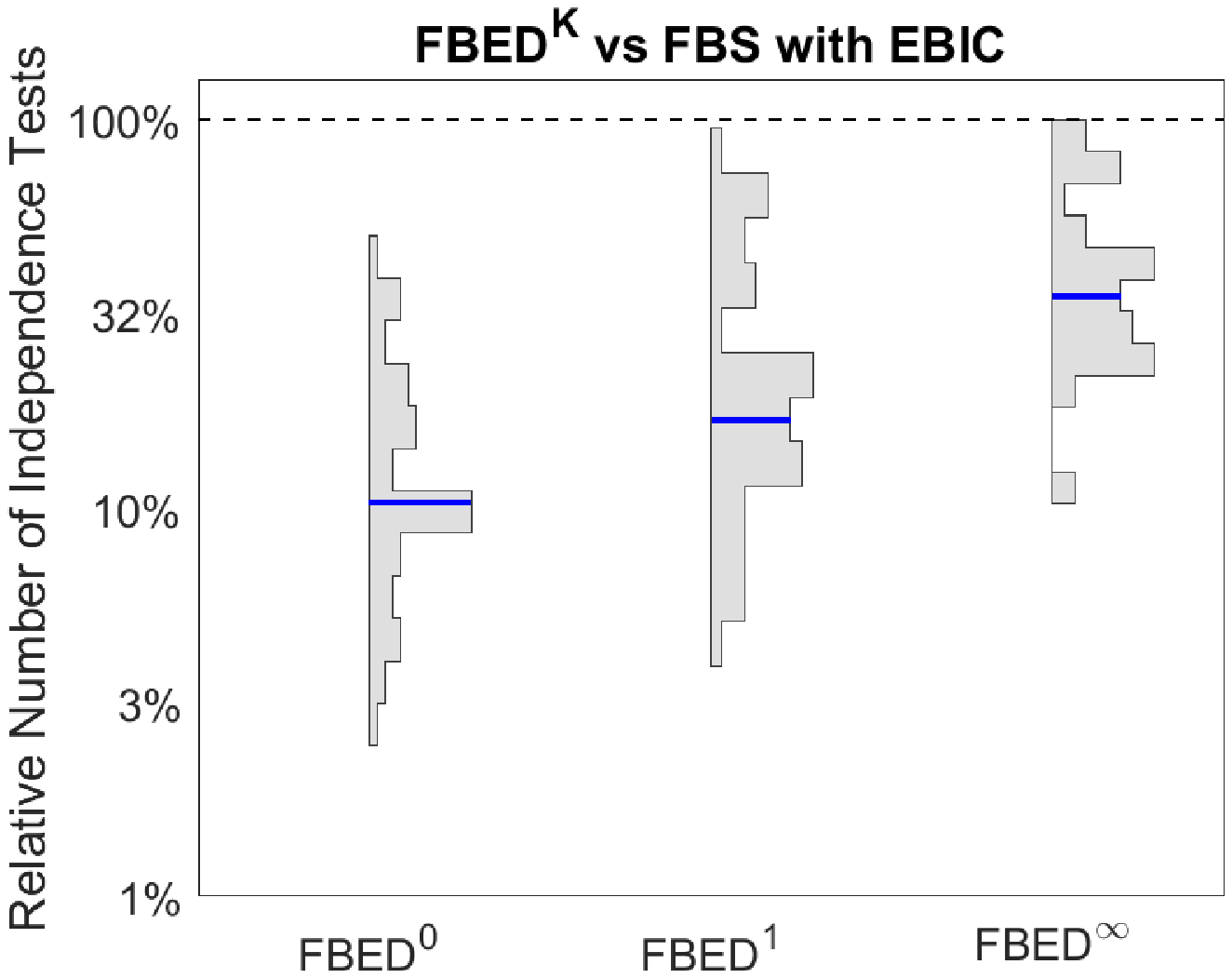}
\end{subfigure}
\caption{
\textbf{Number of Performed Independence Tests:} 
The figure shows the distributions of the relative number of tests performed by each algorithm, in comparison to \FBS{}.
The comparison is between \FBS{} and \fFBS{K} using the same hyper-parameter values and on all datasets.
Values below 100\% indicate that \fFBS{K} performed fewer tests.
The median value on each histogram is shown with a blue line.
Overall, the \fFBS{K} variants perform significantly fewer tests than \FBS{}.
}
\label{fig:fsvsffs_tests}
\end{figure}

\subsection{Comparison with LASSO-FS}
In this section, we compare \fFBS{K} and \FBS{} to LASSO-FS.
For each algorithm and selection criterion, we optimized over all hyper-parameter values.
Note that, this comparison favors LASSO-FS, as it is allowed to use up to 1000 values for $\lambda$, while \fFBS{K} and \FBS{} use only 4 values for each selection criterion.
The main reason for that is that it takes much more time to optimize over many values of $\alpha$ or $\gamma$.
\textit{The main objective of this comparison is to compare \fFBS{K} to \FBS{} and LASSO-FS in a realistic scenario, where hyper-parameter values are optimized.}
Furthermore, we will also compare EBIC with IT.

For a fairer comparison, we performed an additional simulation where we compare the optimized results of \fFBS{K} and \FBS{} to LASSO-FS, by fixing the minimum number of variables LASSO-FS selects; this was done for each algorithm, selection criterion and hyper-parameter value separately.
\textit{This simulation aims at showing how \fFBS{K} and \FBS{} compare to LASSO-FS when selecting the same number of variables, or in other words, how well each algorithm orders the variables.}

\subsection*{Comparison when optimizing over hyper-parameters}
\label{sec:exp:fs:opt}

\setlength{\tabcolsep}{.25em}
\begin{table*}[!t]
\centering
	\caption{
	Area under the ROC curve, classification accuracy and number of selected variables for each \fFBS{K} method and \FBS{} using independence tests or EBIC and LASSO-FS with 1000 $\lambda$ values. 
	The results are obtained after optimizing the hyper-parameters of the feature selection and modeling algorithms. 
	Bold and italic entries denote that the method is significantly better or worse than all other methods respectively. 
	The score is the average rank of each method over all datasets and the final rank is computed using those scores.}
    \label{tbl:opt}
  \fontsize{6pt}{8pt}\selectfont
  \begin{tabular}{ccl|*{12}{c}|cc}
    \toprule
    &&
    Algorithm &
    \text{musk} &
    \text{sylva}&
    \text{madelon} &
    \text{secom} &
    \text{gina} &
    \text{hiva} &
    \text{gisette} &
    \text{p53} &
    \text{arcene} &
    \text{nova} &
    \text{dexter} &
    \text{dorothea} &
    \text{Score} &
    \text{Rank}
\\
\midrule
\multirow{9}{*}{\rotatebox{90}{AUC}}
&& \fFBS{0} - IT & 93.0 & 99.9 & 63.4 & 67.8 & 93.4 & 69.3 & 99.4 & 95.0 & 78.6 & 93.7 & 97.0 & 84.3 & 3.75 & 2 \\
&& \fFBS{1} - IT & 94.7 & 99.9 & 63.1 & 66.3 & 93.6 & 69.2 & 99.4 & 95.0 & 79.0 & 94.0 & 96.9 & 85.0 & 4.33 & 4 \\
&& \fFBS{\infty} - IT & 96.9 & 99.9 & 63.1 & 66.4 & 93.8 & 68.6 & 99.4 & 94.7 & 78.8 & 94.2 & 97.0 & 84.9 & 4.75 & 5 \\
&& \FBS{} - IT & 96.9 & 99.9 & 63.1 & 68.0 & 93.7 & 70.3 & 99.4 & 95.1 & 77.6 & 93.2 & 96.4 & 84.5 & 4.08 & 3 \\
&& \fFBS{0} - EBIC & \textit{90.9} & 99.9 & 63.4 & 66.3 & \textit{92.7} & 69.5 & 99.2 & 93.9 & 77.1 & 92.6 & 97.1 & 81.4 & 6.33 & 6 \\
&& \fFBS{1} - EBIC & 92.8 & 99.9 & 63.3 & 64.4 & 93.1 & 69.1 & 99.3 & 94.8 & 77.6 & 93.5 & 97.0 & 83.2 & 6.63 & 8 \\
&& \fFBS{\infty} - EBIC & 95.7 & 99.9 & 63.3 & 65.3 & 93.3 & 68.8 & 99.3 & 95.0 & 76.7 & 93.2 & 96.9 & 83.4 & 6.38 & 7 \\
&& \FBS{} - EBIC & 95.6 & 99.9 & 63.3 & 65.2 & 93.3 & 68.7 & 99.4 & 94.5 & 76.3 & 92.4 & 96.4 & 82.4 & 7.33 & 9 \\
&& LASSO-FS & \textbf{97.3} & 99.9 & 63.3 & 69.2 & \textbf{94.2} & 70.3 & \textbf{99.6} & 96.3 & \textbf{83.7} & \textbf{96.1} & \textbf{98.2} & \textbf{88.1} & 1.42 & 1 \\
\midrule
\multirow{9}{*}{\rotatebox{90}{Accuracy}}
&& \fFBS{0} - IT & 91.8 & 99.3 & 60.6 & 93.1 & 86.7 & 96.7 & 97.1 & 99.2 & 71.9 & 91.5 & 90.8 & 92.7 & 4.38 & 2 \\
&& \fFBS{1} - IT & 92.9 & 99.3 & 60.3 & 93.1 & 86.9 & 96.7 & 97.2 & 99.2 & 72.3 & 91.5 & 90.4 & 92.7 & 4.42 & 3 \\
&& \fFBS{\infty} - IT & 94.5 & 99.3 & 60.5 & 93.1 & 87.0 & 96.7 & 97.0 & 99.2 & 71.7 & 91.4 & 90.3 & 92.6 & 5.58 & 5 \\
&& \FBS{} - IT & 94.5 & 99.3 & 60.4 & 93.1 & 87.0 & 96.7 & 97.0 & 99.2 & 70.6 & 91.5 & 89.9 & 92.7 & 4.79 & 4 \\
&& \fFBS{0} - EBIC & \textit{89.4} & 99.2 & 60.4 & \textbf{93.3} & \textit{85.7} & 96.7 & 96.5 & 99.2 & 71.3 & 91.1 & 90.7 & 93.0 & 5.75 & 7 \\
&& \fFBS{1} - EBIC & 91.9 & 99.2 & 60.5 & 93.2 & 86.1 & 96.7 & 96.7 & 99.2 & 70.4 & 91.0 & 90.7 & 92.7 & 5.83 & 8 \\
&& \fFBS{\infty} - EBIC & 93.5 & 99.2 & 60.5 & 93.2 & 86.5 & 96.7 & 96.9 & 99.2 & 69.9 & 90.8 & 90.5 & 92.7 & 5.75 & 7 \\
&& \FBS{} - EBIC & 93.4 & 99.2 & 60.5 & 93.2 & 86.4 & 96.7 & 96.9 & 99.2 & 68.9 & 91.2 & 90.1 & 92.8 & 6.13 & 9 \\
&& LASSO-FS & \textbf{94.9} & \textbf{99.4} & 61.0 & 93.0 & 87.2 & 96.7 & \textbf{97.8} & 99.2 & \textbf{76.4} & 90.6 & \textbf{93.2} & 93.1 & 2.38 & 1 \\
\midrule
\multirow{9}{*}{\rotatebox{90}{Selected Variables}}
&& \fFBS{0} - IT & 23.4 & 17.9 & 8.2 & 12.4 & 36.8 & 22.2 & 73.0 & 24.3 & 9.4 & 65.2 & 17.3 & 18.4 & 4.58 & 5 \\
&& \fFBS{1} - IT & 35.4 & 20.8 & 12.2 & 16.9 & 50.2 & 34.9 & 89.3 & 45.4 & 10.1 & 72.2 & 21.0 & 21.4 & 6.25 & 6 \\
&& \fFBS{\infty} - IT & 82.5 & 29.3 & 18.8 & 54.0 & 199.1 & \textit{96.4} & 87.2 & 72.1 & 11.8 & 76.3 & 22.2 & 23.1 & 8.21 & 8 \\
&& \FBS{} - IT & 78.4 & 29.3 & 18.0 & 44.5 & 161.6 & 59.9 & 79.6 & 59.5 & 10.6 & 65.9 & 17.9 & 20.1 & 6.96 & 7 \\
&& \fFBS{0} - EBIC & \textbf{13.9} & \textbf{11.0} & \textbf{2.7} & \textbf{3.3} & \textbf{23.1} & 6.1 & \textbf{42.6} & \textbf{12.1} & 7.7 & \textbf{40.7} & \textbf{13.7} & 8.9 & 1.17 & 1 \\
&& \fFBS{1} - EBIC & 21.5 & 11.8 & 3.7 & 4.0 & 27.5 & 7.1 & 53.4 & 16.5 & 8.0 & 51.6 & 15.6 & 13.5 & 2.33 & 2 \\
&& \fFBS{\infty} - EBIC & 37.9 & 12.1 & 4.0 & 5.3 & 32.7 & 8.2 & 73.9 & 19.3 & 8.4 & 55.7 & 17.4 & 14.2 & 3.88 & 4 \\
&& \FBS{} - EBIC & 37.8 & 12.1 & 4.0 & 5.8 & 32.9 & 9.3 & 75.8 & 19.3 & 6.7 & 51.1 & 15.2 & 7.4 & 3.21 & 3 \\
&& LASSO-FS & \textit{128.7} & \textit{73.2} & 13.3 & 21.0 & 164.5 & 43.7 & \textit{275.5} & \textit{114.6} & \textit{37.5} & \textit{314.4} & \textit{189.2} & \textit{36.9} & 8.42 & 9 \\
\bottomrule
  \end{tabular}
\end{table*}

A summary of the results of the first simulation study, averaged over repetitions, measuring the AUC, classification accuracy and number of selected variables is shown in Table~\ref{tbl:opt}.
To summarize the results, we computed the score of each algorithm, which is as the average rank of that algorithm over all datasets.
The rank of each algorithm is then computed based on its score.
Algorithms that are statistically significantly better than all others are shown in bold, whereas algorithms that are worse than the rest are shown in italic.
To test for statistical significance, we employed a bootstrap-based procedure.
Specifically, we used bootstrapping to compute the probability that algorithm $A_i$ is better/worse or equal than all others in terms of some measure of interest $f$ (for example, AUC), that is $P(\wedge_{j \neq i} f(A_i) \geq f(A_j))$.
An algorithm is considered the best/worst if it is so with probability at least $0.95$.
The procedure is described next.
Let $f_{i,k}$ denote the measure of interest of algorithm $i$ on test set $k$.
We resample with replacement $B = 100K$ times the test sets and compute $f$, denoted as $f_{i,k,b}$ for the b-th sample of algorithm $i$ and test set $k$.
Then, $f_{i,k,b}$ are averaged over test sets, obtaining $\hat{f}_{i,b}$.
The probability $P(\wedge_{j \neq i} f(A_i) \geq f(A_j))$ is then computed as $1/B \sum_b I(\hat{f}_{i,b} \geq \max_j \hat{f}_{j,b})$, where $I$ takes the value 1 if the expression inside it evaluates to true. 

LASSO-FS has the best predictive performance, being statistically significantly better in 7 and 5 datasets in terms of AUC and classification accuracy respectively. 
The difference is noticeable in datasets with many variables ($\geq$ 5000), and mostly for the arcene, dexter, and dorothea datasets.
It does however select more variables, being statistically significantly worse in 8 out of 12 datasets, including most datasets it performed better.

\FBS{} and \fFBS{K} perform similarly on all datasets, with the independence test based variants performing slightly better than the EBIC ones.
In terms of number of selected variables, IT-based variants tend to select more variables than EBIC-based ones.
Note that, this result highly depends on the hyper-parameter values used for the experiments.
EBIC is more limited in that aspect, as it does not allow one to select arbitrarily many variables, in contrast to IT, which tends to selects more variables with higher significance level.
In general, methods that select more variables tend to perform better in our experiment, as expected.

Overall, there is no clear winner, and the choice depends solely on the goal.
If the goal is predictive performance, LASSO-FS is clearly preferable.
If on the other hand one is interested in interpretability, then \fFBS{0} or \fFBS{1} are the best choices, with the latter being especially attractive due to its theoretical properties.
Furthermore, even though they do not perform as well as LASSO-FS, they still work reasonably well, with the additional advantage that they are the fastest among all forward-selection-based methods.
For instance, \fFBS{0} with independence tests selects relatively few variables while being ranked 2nd in terms of AUC and accuracy.
To summarize, \textit{if the task is predictive performance, LASSO-FS is to be preferred, while if the focus is interpretability, \fFBS{0} and \fFBS{1} are the best choices}.

We must note that those results are somewhat artificial, as the performance of \fFBS{K} and \FBS{} highly depends on the hyper-parameter values chosen for the experiment, while LASSO-FS is not as sensitive to those choices.
Furthermore, the fact that hyper-parameters are optimized based on performance naturally tends to favor methods that select more variables.
Therefore, the procedure puts LASSO-FS at a disadvantage in terms of interpretability.
The experiments presented next, comparing the performance of all algorithms while selected the same number of variables, were performed for exactly those reasons.

\subsection*{Limiting the number of selected variables}
\begin{figure}[t!]
\begin{subfigure}[t]{0.475\textwidth}
\centering
\includegraphics[width=\textwidth]{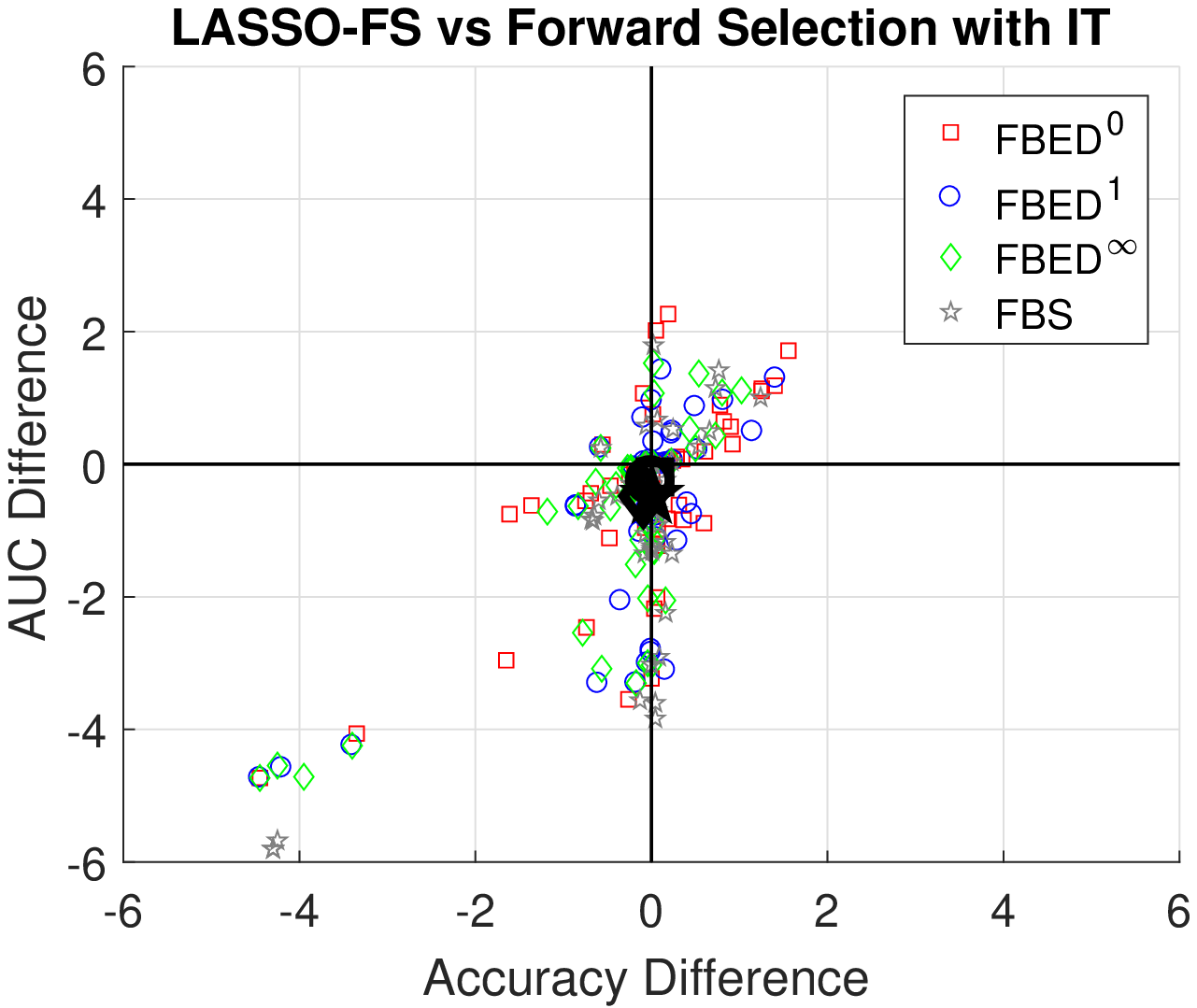}
%\caption{}
\end{subfigure}
~
\begin{subfigure}[t]{0.475\textwidth}
\centering
\includegraphics[width=\textwidth]{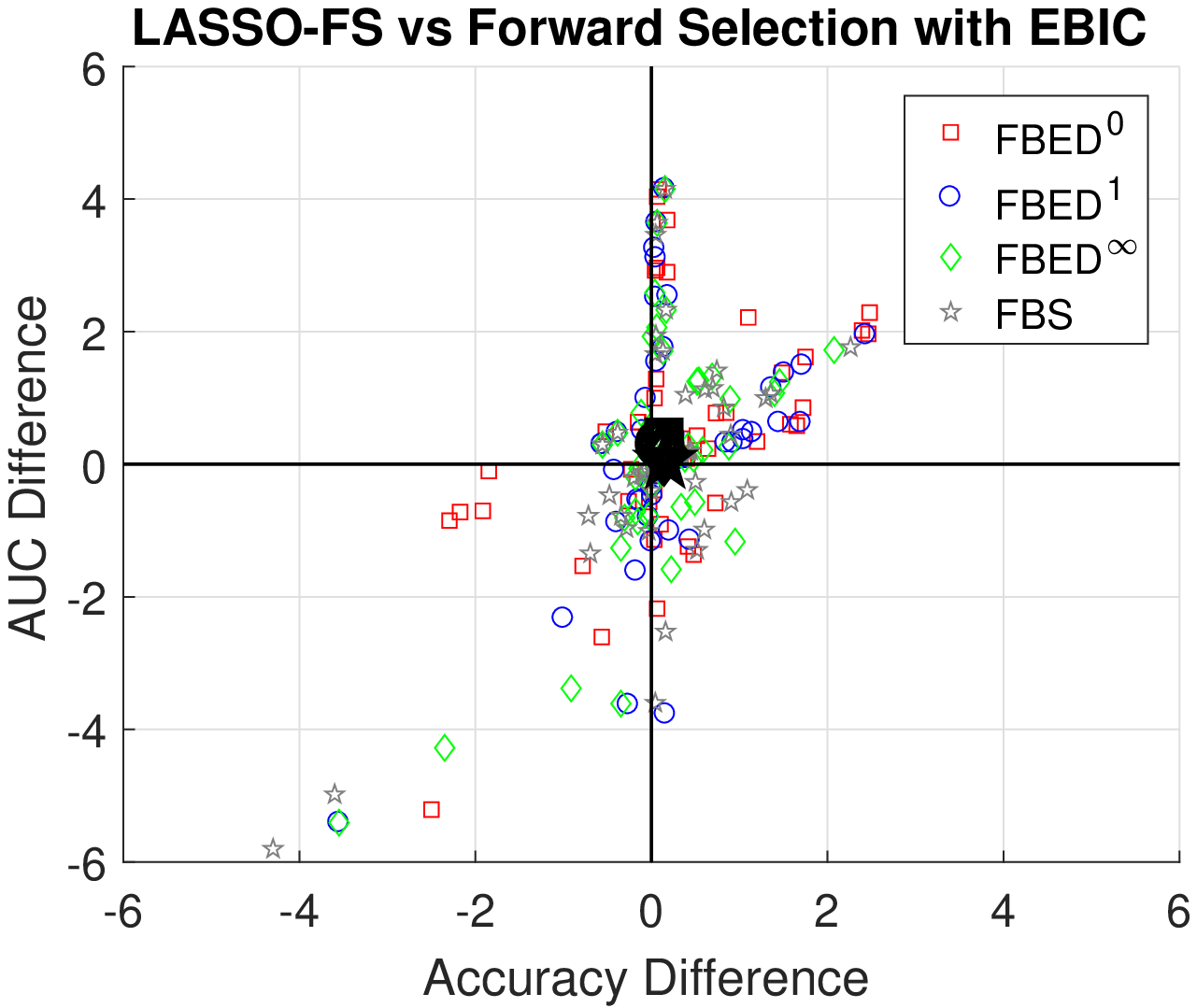}
%\caption{}
\end{subfigure}
\caption{
\textbf{LASSO-FS with limit on selected variables:} 
The x-axis and y-axis show the difference in accuracy and AUC respectively of LASSO-FS and forward selection based variants, with positive values indicating that LASSO-FS performs worse.
Black points show the median AUC and accuracy for each algorithm.
Overall, all algorithms perform similarly if the number of variables to select is limited.
}
\label{fig:lassovsfs}
\end{figure}

We compare \FBS{} and \fFBS{K} to LASSO-FS, when LASSO-FS is forced to select a fixed number of variables $M$.
We will refer to forward selection based methods to FS hereafter.
LASSO-FS with a limit on the variables was compared to each selection criterion and hyper-parameter of \FBS{} and \fFBS{K} (for example, \FBS{} with IT and $\alpha = 0.05$ vs LASSO-FS selecting the same number of variables).

The results are shown in Figure~\ref{fig:lassovsfs}.
The x-axis and y-axis correspond to the difference in accuracy and AUC respectively between LASSO-FS and FS, that is, Performance(FS) - Performance(LASSO-FS).
The black points correspond to the median accuracy and AUC difference for each algorithm.
Points in the upper-right corner are cases where FS dominates LASSO-FS in terms of both performance measures, while the opposite holds for points in the lower-left corner.
It can be seen that all methods perform similarly, with the median values falling very close to the center.
The few cases where LASSO-FS performs better correspond to the arcene dataset.
Note that, arcene is the dataset which contains the fewest number of samples, while also containing a large number of variables.
It has 200 samples and 10000 variables, and only 120/160 are used for training and validation respectively.
Arcene was also the dataset in which LASSO-FS outperformed the rest by a large margin in the previous experiments.
Theoretical results by \cite{Ng2004} show that LASSO-FS performs well in settings with low sample size and many irrelevant variables, as is the case for the arcene dataset.
It is not clear how forward selection based procedures perform in such cases, and whether this is simply an outlier or LASSO-FS is inherently superior.
It would be interesting to study this effect in more depth, but it is out of the scope of the current paper.
In summary, \textit{LASSO-FS works as well as \FBS{} and \fFBS{K} when the number of variables to select is the same}.

\subsection{Discussion}
We compared \FBS{}, \fFBS{K} and LASSO-FS across a variety of datasets and settings.
An interesting result is that all algorithms perform about equally well, when limited to select the same number of variables.
The main advantage of LASSO-FS is that it is easier and faster to tune, as results for different values of $\lambda$ can be obtained with a single run, at least for many important cases such as generalized linear models \citep{Friedman2010}.
Its main drawback is that it often requires specialized algorithms and treatment for different problems \citep{Meier2008, VanDeGeer2011, Ivanoff2016}, which may be non-convex  \citep{VanDeGeer2011} and computationally demanding \citep{Fan2010}.

\section{Conclusion}
We presented a heuristic to speed-up the forward-backward feature selection algorithm, which gives rise to a family of algorithms, called forward-backward selection with early dropping (\fFBS{K}).
We investigated the theoretical properties of three of its members, namely \fFBS{0}, \fFBS{1} and \fFBS{\infty}.
We show that, if the distribution of the data can be faithfully represented by a Bayesian network or maximal ancestral graph, \fFBS{1} and \fFBS{\infty} respectively can identify the Markov blanket of the target variable.
In experiments on real data we show that \fFBS{K} is significantly faster than \FBS{}, while performing similarly or better in terms of predictive performance.
Furthermore, \fFBS{K} and \FBS{} perform very similar to LASSO, when restricted to select the same number of variables, while being much more general.
Overall, among all forward-selection based algorithms, \fFBS{0} and \fFBS{1} offer the best trade-off in terms of predictive performance, running time and number of selected features.

\section*{Acknowledgments}
We would like to thank Vincenzo Lagani and Michalis Tsagris for their helpful comments.
The research leading to these results has received funding from the European Research Council
under the European Union's Seventh Framework Programme (FP/2007-2013) / ERC Grant
Agreement n. 617393.

\appendix

\section{Proofs}\label{app:theory}

We proceed by listing some axioms about conditional independence \citep{Pearl2000}, called \textbf{semi-graphoid} axioms, which will be useful later on.
Those axioms are general, as they hold for any probability distribution.
For all of the proofs we assume that the algorithms have access to an \textbf{independence oracle} that can perfectly determine whether a given conditional dependence or independence holds.
Furthermore, in all proofs we will use the terms d-connected/m-connected (d-separated/m-separated) and dependent (independent) interchangeably; this is possible due to the faithfulness assumption. 
\\
\\
\begin{centering}
\begin{tabular}{|l l|}
\hline
\textbf{Symmetry} &
$\condind{\mathbf{X}}{\mathbf{Y}}{\mathbf{Z}}
\Rightarrow
\condind{\mathbf{Y}}{\mathbf{X}}{\mathbf{Z}}$
\\
\textbf{Decomposition} &
$\condind{\mathbf{X}}{\mathbf{Y}\cup\mathbf{W}}{\mathbf{Z}}
\Rightarrow
\condind{\mathbf{X}}{\mathbf{Y}}{\mathbf{Z}}
\wedge
\condind{\mathbf{X}}{\mathbf{W}}{\mathbf{Z}}
$
\\
\textbf{Weak Union} &
$\condind{\mathbf{X}}{\mathbf{Y}\cup\mathbf{W}}{\mathbf{Z}}
\Rightarrow
\condind{\mathbf{X}}{\mathbf{Y}}{\mathbf{Z}\cup\mathbf{W}}
$
\\
\textbf{Contraction} &
$\condind{\mathbf{X}}{\mathbf{Y}}{\mathbf{Z}}
\wedge
\condind{\mathbf{X}}{\mathbf{W}}{\mathbf{Y}\cup\mathbf{Z}}
\Rightarrow
\condind{\mathbf{X}}{\mathbf{Y}\cup\mathbf{W}}{\mathbf{Z}}
$
\\
\hline
\end{tabular}
\end{centering}
\\

\noindent Using those axioms we prove the following lemma.

\begin{lemma}\label{lemma_ind}
Let $A$, $T$ be variables and $\mathbf{B}$, $\mathbf{C}$ sets of variables.
Then \\
$\condind{T}{A}{\textbf{B}\cup\mathbf{C}} \wedge \condind{T}{\textbf{B}}{\mathbf{C}} \Rightarrow \condind{T}{A}{\mathbf{C}}$ holds for any such variables.
\end{lemma}
\begin{proof}

\begin{flalign*}
&\condind{T}{A}{\textbf{B}\cup\mathbf{C}} \wedge \condind{T}{\textbf{B}}{\mathbf{C}} \Rightarrow & \text{ (\textbf{Contraction})} \\
&\condind{T}{A \cup \textbf{B}}{\mathbf{C}} \Rightarrow & \text{ (\textbf{Decomposition})} \\
&\condind{T}{A}{\mathbf{C}} \wedge \condind{T}{\textbf{B}}{\mathbf{C}}
\end{flalign*}
\end{proof}

\noindent The following lemma will be useful for proving some of the theorems.

\begin{lemma}\label{lemma_incr}
Let $\mathbf{V_{sel}}$ be a set of variables selected for some target $T$ and $\mathbf{V_{rem}} = \mathbf{V_{\mathcal{D}}} \setminus \mathbf{V_{sel}}$.
Assume that $\forall V_r \in \mathbf{V_{rem}} \ \condind{T}{V_r}{\mathbf{V_{sel}}}$ holds.
Then, if $\exists V_s \in \mathbf{V_{sel}}$ such that $\condind{T}{V_s}{\mathbf{V_{sel}} \setminus V_s}$ holds, $\forall V_r \in \mathbf{V_{rem}} \ \condind{T}{V_r}{\mathbf{V_{sel}} \setminus V_s}$ also holds.
\end{lemma}

\begin{proof}
We are given that $\forall V_r \in \mathbf{V_{rem}} \ \condind{T}{V_r}{\mathbf{V_{sel}}}$ holds.
By applying Lemma \ref{lemma_ind} to each variable in $V_r \in \mathbf{V_{rem}}$ with $A = V_r$, $\mathbf{B} = \{V_s\}$ and $\mathbf{C} = \mathbf{V_{sel}} \setminus V_s$, we get that $\condind{T}{V_r}{V_s \cup (\mathbf{V_{sel}} \setminus V_s)} \wedge \condind{T}{V_s}{\mathbf{V_{sel}} \setminus V_s} \Rightarrow \condind{T}{V_r}{\mathbf{V_{sel}} \setminus V_s}$ holds for any such $V_r$, which concludes the proof.
\end{proof}

To put it simple, Lemma \ref{lemma_incr} states that if we remove any variable $V_s$ from a set of selected variables $\mathbf{V_{sel}}$  by conditioning on $\mathbf{V_{sel}} \setminus V_s$, no variable that is not in $\mathbf{V_{sel}}$ becomes conditionally dependent with $T$ given $\mathbf{V_{sel}} \setminus V_s$.
In practice this means that removing variables using backward selection from a set of variables selected by forward selection will not create any additional conditional dependencies, meaning that we do not have to reconsider them again.
\\
\\
\noindent\textbf{Proof of Theorem~\ref{thm:fbs}}
\begin{proof}
To show that $\mathbf{V_{sel}}$ is minimal, we have to show the following
\begin{enumerate}[i]
\item $\forall V_s \in \mathbf{V_{sel}} \ \conddep{T}{V_s}{\mathbf{V_{sel}} \setminus V_s}$ (\textit{No variable can be removed})
\item $\forall V_r \in \mathbf{V_\mathcal{D}} \setminus \mathbf{V_{sel}}, \condind{T}{V_r}{\mathbf{V_{sel}}}$ (\textit{No variable can be added})
\end{enumerate}

\noindent\textbf{Proof of (i)}:
This holds trivially, as backward selection removes any variable $V_s \in \mathbf{V_{sel}}$ if $\condind{T}{V_s}{\mathbf{V_{sel}} \setminus V_s}$ holds.

\noindent\textbf{Proof of (ii)}:
We know that after the termination of forward selection, no variable can be added, that is, $\forall V_r \in \mathbf{V_{rem}} \ \condind{T}{V_r}{\mathbf{V_{sel}}}$ holds.
Given that, Lemma~\ref{lemma_incr} can be repeatedly applied after each variable removal by backward selection, and thus no variable in $\mathbf{V_{rem}}$ can be added to $\mathbf{V_{sel}}$.
\end{proof}

\noindent\textbf{Proof of Theorem~\ref{thm:ffbs}}
\begin{proof}
As is the case with \FBS{}, the forward selection phase of \fFBS{\infty} stops if no more variables can be included.
Using this fact, the proof is identical to the one of Theorem \ref{thm:fbs}.
\end{proof}

\noindent\textbf{Proof of Theorem~\ref{thm:ffbs1mb}}
\begin{proof}
In the first run of \fFBS{1}, all variables that are adjacent to $T$ (that is, its parents and children) will be selected, as none of them can be d-separated from $T$ by any set of variables.
In the next run, all variables connected through a collider path of length 2 (that is, the spouses of $T$) will become d-connected with $T$, since the algorithm conditions on all selected variables (including its children), and thus will be selected.
The resulting set of variables includes the Markov blanket of $T$, but may also include additional variables.
Next we show that all additional variables will be removed by the backward selection phase.
Let MB($T$) be the Markov blanket of $T$ and $\mathbf{S_{ind}} = \mathbf{S} \setminus $MB($T$) be all selected variables not in the Markov blanket of $T$.
By definition, $\condind{T}{\mathbf{X}}{MB(T)}$ holds for any set of variables $\mathbf{X}$ not in MB($T$), and thus also for variables $\mathbf{S_{ind}}$.
By applying the weak union graphoid axiom, one can infer that $\forall S_i \in \mathbf{S_{ind}}, \condind{T}{S_i}{MB(T) \cup \mathbf{S_{ind}} \setminus S_i}$ holds, and thus some variable $S_j$ will be removed in the first iteration.
Using the same reasoning and the definition of a Markov blanket, it can be shown that all variables in $\mathbf{S_{ind}}$ will be removed from MB($T$) at some iteration.
To conclude, it suffices to use the fact that variables in MB($T$) will not be removed by the backward selection, as they are not conditionally independent of $T$ given the remaining variables in MB($T$).
\end{proof}

\noindent\textbf{Proof of Theorem~\ref{thm:ffbsinfmb}}
\begin{proof}
In the first run of \fFBS{\infty}, all variables that are adjacent to $T$ (that is, its parents, children and variables connected with $T$ by a bi-directed edge) will be selected, as none of them can be m-separated from $T$ by any set of variables.
After each run additional variables may become admissible for selection.
Specifically, after $k$ runs all variables that are connected with $T$ by a collider path of length $k$ will become m-connected with $T$, and thus will be selected; we prove this next.
Assume that after $k$ runs all variables connected with $T$ by a collider path of length at most $k-1$ have been selected.
By conditioning on all selected variables, all variables that are into some selected variable connected with $T$ by a collider path will become m-connected with $T$.
This is true because conditioning on a variable $Y$ in a collider $\langle X, Y, Z \rangle$ m-connects $X$ and $Z$.
By applying this on each variable on some collider path, it is easy to see that its end-points become m-connected.
Finally, after applying the backward selection phase, all variables that are not in the Markov blanket of $T$ will be removed; the proof is identical to the one used in the proof of Theorem~\ref{thm:ffbs1mb} and thus will be omitted.
\end{proof}

\section{Additional Results} \label{app:results}

\subsection*{Detailed Results of \FBS{} vs \fFBS{K}}
Tables~\ref{tbl:auc}, \ref{tbl:acc} and \ref{tbl:vars} show the detailed results of all algorithms in terms of area under the ROC curve (AUC), classification accuracy (ACC) and number of selected variables respectively.
The results correspond to the ones shown in Figure~\ref{fig:fsvsffs_perf}. 
The values $\gamma = $def corresponds to default value, that is $\gamma = 1 - 0.5 \cdot \log(n) / \log(p)$.
Larger values of $\alpha$ and lower values of $\gamma$ tend to perform better, as they lead to the selection of more variables.
The number of variables selected using the EBIC criterion is usually between the number selected by the IT criterion with $\alpha = 0.001$ and $\alpha = 0.01$, irrespective of the $\gamma$ value used.

\subsection*{LASSO-PM vs Logistic Regression}
Figure~\ref{fig:l1vslr} show the comparison of LASSO-PM and standard logistic regression when used as final predictive models.
The algorithms are compared both, in terms of accuracy and AUC.
Each point corresponds to the difference in AUC and accuracy between LASSO-PM and logistic regression.
It can clearly be seen that LASSO-PM performs better overall.
For \FBS{} and \fFBS{\infty} with IT the differences are larger, as those algorithms selected the most variables overall.
The effect isn't as large for EBIC or \fFBS{0}, which both tend to select fewer variables.
We also computed the Spearman correlation between the number of selected variables and difference in performance by pooling all results together, and found correlations of 0.51 and 0.22 for AUC and accuracy respectively, suggesting that there is a significant positive correlation between number of selected variables and difference in performance between both modeling methods.
Thus, \textit{the more variables are selected, the worse standard logistic regression performs compared to LASSO-PM}.

\subsection*{Timing Results of \FBS{}, \fFBS{K} and LASSO}
Table~\ref{tbl:timing} shows the running time of each feature selection algorithm and configuration, on all datasets.
The values correspond to a single run on the complete dataset.
For LASSO-FS we used two values for the maximum number of $\lambda$ values to try, 100 and 1000.
All runs were performed on a single machine, and no runs were performed simultaneously.
It can clearly be seen that LASSO-FS is the fastest in large datasets, irrespective of the number of $\lambda$ values used.
For smaller datasets (musk, sylva, madelon, secomd, gina and hiva), \fFBS{0} and \fFBS{1} are at least as fast as LASSO-FS, and are often even faster.
Note however that the differences can largely be attributed to the implementations of the algorithms.
For LASSO-FS the glmnet implementation was used, which is highly optimized and written in FORTRAN.
In contrast, for \fFBS{K} and \FBS{} we used a custom logistic regression implementation written in MATLAB.
A difference of 1-2 orders of magnitude can be expected between the same implementation in a low-level language such as FORTRAN, C or C++ and higher-level languages such as MATLAB.
Therefore, we would expect that such an implementation would perform similarly to LASSO-FS.
Of course, LASSO-FS has the advantage that it returns the whole solution path, and thus would still be faster in practice if hyper-parameter optimization is also performed.
\fFBS{K} and \FBS{} can be directly compared to each other, as they both use the same implementations.
We can see that \fFBS{0} is up to 3 orders of magnitude faster than \FBS{} (for example, hiva with IT and $\alpha = 0.1$), and usually around 1-2 orders faster, while \fFBS{\infty} is typically around 1 order of magnitude faster than \FBS{}.

%% AUC TABLE L1
\setlength{\tabcolsep}{.25em}
{\renewcommand{\arraystretch}{1.2}
\begin{table*}[!ht]
\centering
	\caption{Area under the ROC curve of each feature selection, selection criterion and hyper-parameter using LASSO-PM. The first four groups use an independence test with significance level $\alpha$, whereas the last four groups use the EBIC criterion with parameter $\gamma$.}
    \label{tbl:auc}
  \scriptsize
  \begin{tabular}{ccl|*{12}{c}}
    \toprule
    &&
    Algorithm &
    \text{musk} &
    \text{sylva}&
    \text{madelon} &
    \text{secom} &
    \text{gina} &
    \text{hiva} &
    \text{gisette} &
    \text{p53} &
    \text{arcene} &
    \text{nova} &
    \text{dexter} &
    \text{dorothea} 
\\
\midrule
\multirow{4}{*}{\rotatebox{90}{$\alpha$ = 0.001}}
&& \fFBS{0}& 90.8 & 99.9 & 63.4 & 64.4 & 92.4 & 67.6 & 99.2 & 94.1 & 76.2 & 91.0 & 97.0 & 84.6 \\
&& \fFBS{1}& 92.3 & 99.9 & 63.3 & 64.2 & 92.9 & 69.1 & 99.3 & 94.8 & 77.9 & 92.0 & 97.0 & 84.8 \\
&& \fFBS{\infty}& 95.7 & 99.9 & 63.3 & 64.0 & 93.1 & 69.3 & 99.3 & 95.1 & 78.5 & 92.5 & 96.7 & 85.1 \\
&& \FBS{}& 95.4 & 99.9 & 63.3 & 64.0 & 92.9 & 69.3 & 99.4 & 94.4 & 77.6 & 92.4 & 96.3 & 84.5 \\
\hline
\multirow{4}{*}{\rotatebox{90}{$\alpha$ = 0.01}}
&& \fFBS{0}& 91.9 & 99.9 & 63.2 & 67.2 & 92.9 & 69.4 & 99.3 & 94.7 & 78.0 & 92.7 & 97.2 & 84.4 \\
&& \fFBS{1}& 93.9 & 99.9 & 63.1 & 65.3 & 93.3 & 69.3 & 99.3 & 95.3 & 78.4 & 93.9 & 97.1 & 84.7 \\
&& \fFBS{\infty}& 96.1 & 99.9 & 63.1 & 66.2 & 93.4 & 69.0 & 99.3 & 94.7 & 78.4 & 94.2 & 97.1 & 84.7 \\
&& \FBS{}& 96.1 & 99.9 & 63.1 & 66.5 & 93.3 & 69.1 & 99.4 & 94.7 & 77.6 & 93.6 & 96.5 & 84.5 \\
\hline
\multirow{4}{*}{\rotatebox{90}{$\alpha$ = 0.05}}
&& \fFBS{0}& 92.6 & 99.9 & 63.2 & 67.8 & 93.2 & 69.5 & 99.4 & 94.8 & 78.7 & 94.0 & 96.9 & 84.5 \\
&& \fFBS{1}& 94.6 & 99.9 & 63.3 & 66.1 & 93.5 & 69.4 & 99.4 & 94.9 & 78.6 & 94.3 & 96.7 & 85.0 \\
&& \fFBS{\infty}& 96.7 & 99.9 & 63.2 & 67.9 & 93.6 & 69.1 & 99.4 & 94.4 & 78.6 & 94.5 & 96.7 & 85.0 \\
&& \FBS{}& 96.7 & 99.9 & 63.2 & 68.3 & 93.5 & 69.5 & 99.4 & 95.2 & 77.6 & 93.8 & 96.5 & 84.5 \\
\hline
\multirow{4}{*}{\rotatebox{90}{$\alpha$ = 0.1}}
&& \fFBS{0}& 93.0 & 99.9 & 63.2 & 67.7 & 93.4 & 69.7 & 99.4 & 95.5 & 78.5 & 94.4 & 96.8 & 84.9 \\
&& \fFBS{1}& 94.7 & 99.9 & 63.2 & 67.5 & 93.6 & 69.5 & 99.4 & 95.3 & 78.5 & 94.6 & 96.8 & 84.9 \\
&& \fFBS{\infty}& 97.0 & 99.9 & 63.2 & 68.0 & 93.8 & 69.2 & 99.4 & 95.2 & 78.5 & 94.6 & 96.8 & 84.9 \\
&& \FBS{}& 96.9 & 99.9 & 63.2 & 68.6 & 93.6 & 70.6 & 99.4 & 95.2 & 77.6 & 93.8 & 96.5 & 84.5 \\
\hline
\multirow{4}{*}{\rotatebox{90}{$\gamma$ = def}}
&& \fFBS{0}& 90.9 & 99.9 & 63.4 & 64.9 & 92.5 & 67.7 & 99.1 & 93.5 & 71.4 & 88.7 & 95.7 & 81.6 \\
&& \fFBS{1}& 92.7 & 99.9 & 63.4 & 65.6 & 92.9 & 67.7 & 99.2 & 94.0 & 74.7 & 90.0 & 96.3 & 82.9 \\
&& \fFBS{\infty}& 95.8 & 99.9 & 63.4 & 65.6 & 93.0 & 67.7 & 99.3 & 94.6 & 75.3 & 91.5 & 96.5 & 83.0 \\
&& \FBS{}& 95.6 & 99.9 & 63.4 & 65.4 & 92.9 & 67.7 & 99.3 & 95.0 & 74.6 & 91.4 & 96.3 & 83.3 \\
\hline
\multirow{4}{*}{\rotatebox{90}{$\gamma$ = 1}}
&& \fFBS{0}& 90.4 & 99.9 & 63.3 & 63.5 & 92.2 & 67.7 & 99.0 & 92.0 & 68.0 & 86.3 & 93.9 & 77.5 \\
&& \fFBS{1}& 91.3 & 99.9 & 63.2 & 63.1 & 92.5 & 67.7 & 99.1 & 93.5 & 70.8 & 87.5 & 95.5 & 79.1 \\
&& \fFBS{\infty}& 95.4 & 99.9 & 63.2 & 63.1 & 92.9 & 67.7 & 99.1 & 93.6 & 70.1 & 89.0 & 95.8 & 80.2 \\
&& \FBS{}& 95.4 & 99.9 & 63.2 & 63.1 & 92.8 & 67.7 & 99.2 & 93.7 & 70.1 & 89.4 & 95.5 & 80.6 \\
\hline
\multirow{4}{*}{\rotatebox{90}{$\gamma$ = 0.5}}
&& \fFBS{0}& 90.6 & 99.9 & 63.4 & 65.1 & 92.5 & 67.7 & 99.1 & 93.4 & 75.0 & 89.9 & 96.4 & 82.5 \\
&& \fFBS{1}& 92.1 & 99.9 & 63.4 & 65.6 & 92.9 & 67.7 & 99.2 & 94.0 & 76.8 & 90.6 & 96.7 & 82.8 \\
&& \fFBS{\infty}& 96.0 & 99.9 & 63.4 & 65.6 & 93.1 & 67.7 & 99.3 & 94.9 & 77.7 & 91.9 & 96.8 & 83.1 \\
&& \FBS{}& 95.6 & 99.9 & 63.4 & 65.5 & 92.9 & 67.7 & 99.3 & 94.9 & 77.1 & 91.8 & 96.3 & 84.5 \\
\hline
\multirow{4}{*}{\rotatebox{90}{$\gamma$ = 0}}
&& \fFBS{0}& 91.1 & 99.9 & 63.3 & 67.2 & 92.7 & 69.5 & 99.2 & 94.1 & 78.0 & 92.6 & 97.2 & 84.0 \\
&& \fFBS{1}& 93.1 & 99.9 & 63.3 & 65.5 & 93.1 & 69.3 & 99.3 & 95.3 & 78.1 & 93.5 & 97.1 & 84.4 \\
&& \fFBS{\infty}& 96.0 & 99.9 & 63.3 & 66.5 & 93.3 & 69.3 & 99.4 & 95.4 & 78.1 & 93.9 & 97.1 & 84.2 \\
&& \FBS{}& 95.6 & 99.9 & 63.3 & 65.9 & 93.3 & 68.4 & 99.4 & 94.3 & 77.6 & 93.4 & 96.5 & 84.5 \\
\hline
  \end{tabular}
\end{table*}

%% ACCURACY TABLE
\setlength{\tabcolsep}{.25em}
{\renewcommand{\arraystretch}{1.2}
\begin{table*}[!t]
\centering
	\caption{Classification accuracy of each feature selection, selection criterion and hyper-parameter using LASSO-PM. Decisions were made by thresholding probabilities at 50\%. The first four groups use an independence test with significance level $\alpha$, whereas the last four groups use the EBIC criterion with parameter $\gamma$.}
    \label{tbl:acc}
  \scriptsize
  \begin{tabular}{ccl|*{12}{c}}
    \toprule
    &&
    Algorithm &
    \text{musk} &
    \text{sylva}&
    \text{madelon} &
    \text{secom} &
    \text{gina} &
    \text{hiva} &
    \text{gisette} &
    \text{p53} &
    \text{arcene} &
    \text{nova} &
    \text{dexter} &
    \text{dorothea} 
\\
\midrule
\multirow{4}{*}{\rotatebox{90}{$\alpha$ = 0.001}}
&& \fFBS{0}& 89.3 & 99.2 & 60.8 & 93.3 & 85.4 & 96.7 & 96.3 & 99.2 & 68.8 & 90.3 & 90.6 & 92.6 \\
&& \fFBS{1}& 91.7 & 99.2 & 60.4 & 93.3 & 86.1 & 96.7 & 96.6 & 99.2 & 70.8 & 90.7 & 91.2 & 92.6 \\
&& \fFBS{\infty}& 93.3 & 99.2 & 60.4 & 93.3 & 86.1 & 96.7 & 96.9 & 99.2 & 71.0 & 90.6 & 90.4 & 92.2 \\
&& \FBS{}& 93.3 & 99.2 & 60.4 & 93.3 & 86.1 & 96.7 & 97.0 & 99.2 & 70.6 & 91.0 & 89.6 & 92.7 \\
\hline
\multirow{4}{*}{\rotatebox{90}{$\alpha$ = 0.01}}
&& \fFBS{0}& 90.4 & 99.2 & 60.2 & 93.2 & 85.9 & 96.6 & 96.8 & 99.2 & 69.9 & 91.1 & 91.2 & 92.3 \\
&& \fFBS{1}& 92.4 & 99.2 & 60.5 & 93.2 & 86.5 & 96.6 & 96.9 & 99.2 & 70.5 & 91.3 & 90.9 & 92.5 \\
&& \fFBS{\infty}& 93.8 & 99.2 & 60.8 & 93.2 & 86.6 & 96.6 & 97.0 & 99.2 & 70.4 & 91.0 & 90.6 & 92.5 \\
&& \FBS{}& 93.6 & 99.2 & 60.7 & 93.2 & 86.6 & 96.6 & 97.1 & 99.3 & 70.6 & 91.3 & 90.2 & 92.9 \\
\hline
\multirow{4}{*}{\rotatebox{90}{$\alpha$ = 0.05}}
&& \fFBS{0}& 91.3 & 99.2 & 60.6 & 93.1 & 86.4 & 96.7 & 97.1 & 99.3 & 71.0 & 91.3 & 90.6 & 92.8 \\
&& \fFBS{1}& 92.8 & 99.3 & 60.8 & 93.0 & 86.7 & 96.7 & 97.0 & 99.2 & 70.9 & 91.1 & 90.3 & 93.0 \\
&& \fFBS{\infty}& 94.3 & 99.3 & 60.8 & 93.1 & 86.7 & 96.6 & 97.0 & 99.2 & 70.9 & 90.9 & 90.2 & 93.0 \\
&& \FBS{}& 94.1 & 99.3 & 60.8 & 93.2 & 86.9 & 96.7 & 97.1 & 99.2 & 70.6 & 91.2 & 90.3 & 92.9 \\
\hline
\multirow{4}{*}{\rotatebox{90}{$\alpha$ = 0.1}}
&& \fFBS{0}& 91.8 & 99.3 & 60.8 & 93.1 & 86.8 & 96.7 & 97.1 & 99.2 & 70.4 & 91.6 & 90.4 & 92.8 \\
&& \fFBS{1}& 92.9 & 99.3 & 60.9 & 93.1 & 86.9 & 96.7 & 97.2 & 99.2 & 70.4 & 91.4 & 90.3 & 92.8 \\
&& \fFBS{\infty}& 94.5 & 99.3 & 60.9 & 93.1 & 87.0 & 96.7 & 97.2 & 99.2 & 70.4 & 91.2 & 90.3 & 92.8 \\
&& \FBS{}& 94.5 & 99.3 & 60.9 & 93.1 & 86.9 & 96.7 & 97.1 & 99.2 & 70.6 & 91.2 & 90.3 & 92.9 \\
\hline
\multirow{4}{*}{\rotatebox{90}{$\gamma$ = def}}
&& \fFBS{0}& 89.3 & 99.2 & 61.0 & 93.3 & 85.8 & 96.7 & 96.3 & 99.2 & 65.4 & 89.5 & 88.4 & 93.0 \\
&& \fFBS{1}& 92.0 & 99.2 & 60.6 & 93.2 & 86.0 & 96.7 & 96.6 & 99.2 & 67.0 & 90.1 & 90.2 & 92.9 \\
&& \fFBS{\infty}& 93.5 & 99.2 & 60.6 & 93.2 & 86.4 & 96.7 & 96.9 & 99.2 & 68.0 & 90.5 & 90.0 & 92.9 \\
&& \FBS{}& 93.4 & 99.2 & 60.6 & 93.2 & 86.2 & 96.7 & 96.7 & 99.2 & 67.8 & 90.8 & 89.8 & 92.8 \\
\hline
\multirow{4}{*}{\rotatebox{90}{$\gamma$ = 1}}
&& \fFBS{0}& 88.3 & 99.2 & 60.7 & 93.3 & 85.2 & 96.6 & 95.8 & 99.1 & 62.8 & 88.7 & 85.6 & 93.0 \\
&& \fFBS{1}& 91.4 & 99.2 & 60.8 & 93.3 & 85.8 & 96.7 & 96.3 & 99.2 & 65.3 & 89.2 & 87.9 & 92.8 \\
&& \fFBS{\infty}& 93.2 & 99.2 & 60.8 & 93.3 & 86.1 & 96.7 & 96.4 & 99.2 & 65.3 & 89.6 & 89.1 & 92.8 \\
&& \FBS{}& 93.3 & 99.2 & 60.8 & 93.3 & 86.0 & 96.7 & 96.3 & 99.2 & 65.1 & 89.7 & 88.6 & 92.8 \\
\hline
\multirow{4}{*}{\rotatebox{90}{$\gamma$ = 0.5}}
&& \fFBS{0}& 88.6 & 99.2 & 61.1 & 93.3 & 85.7 & 96.7 & 96.2 & 99.2 & 67.0 & 90.1 & 89.7 & 92.6 \\
&& \fFBS{1}& 91.6 & 99.2 & 60.9 & 93.2 & 85.9 & 96.7 & 96.6 & 99.2 & 69.2 & 90.4 & 90.6 & 92.4 \\
&& \fFBS{\infty}& 93.4 & 99.2 & 60.9 & 93.2 & 86.4 & 96.7 & 96.9 & 99.2 & 70.9 & 90.6 & 90.4 & 92.4 \\
&& \FBS{}& 93.4 & 99.2 & 60.9 & 93.2 & 86.1 & 96.7 & 96.7 & 99.2 & 70.1 & 90.8 & 89.8 & 92.6 \\
\hline
\multirow{4}{*}{\rotatebox{90}{$\gamma$ = 0}}
&& \fFBS{0}& 89.5 & 99.2 & 60.5 & 93.2 & 85.7 & 96.7 & 96.6 & 99.2 & 71.9 & 91.1 & 91.0 & 92.4 \\
&& \fFBS{1}& 92.1 & 99.2 & 60.8 & 93.2 & 86.3 & 96.7 & 96.7 & 99.2 & 71.7 & 91.0 & 90.8 & 92.7 \\
&& \fFBS{\infty}& 93.5 & 99.2 & 60.6 & 93.2 & 86.7 & 96.7 & 97.0 & 99.3 & 71.7 & 90.9 & 90.9 & 92.5 \\
&& \FBS{}& 93.3 & 99.2 & 60.7 & 93.2 & 86.6 & 96.7 & 97.1 & 99.2 & 70.6 & 91.4 & 90.2 & 92.9 \\
\hline
  \end{tabular}
\end{table*}

%% #VARS TABLE
\setlength{\tabcolsep}{.25em}
{\renewcommand{\arraystretch}{1.2}
\begin{table*}[!t]
\centering
	\caption{Number of selected variables of each feature selection, selection criterion and hyper-parameter using LASSO-PM. The first four groups use an independence test with significance level $\alpha$, whereas the last four groups use the EBIC criterion with parameter $\gamma$.}
    \label{tbl:vars}
  \scriptsize
  \begin{tabular}{ccl|*{12}{r}}
    \toprule
    &&
    Algorithm &
    \text{musk} &
    \text{sylva}&
    \text{madelon} &
    \text{secom} &
    \text{gina} &
    \text{hiva} &
    \text{gisette} &
    \text{p53} &
    \text{arcene} &
    \text{nova} &
    \text{dexter} &
    \text{dorothea} 
\\
\midrule
\multirow{4}{*}{\rotatebox{90}{$\alpha$ = 0.001}}
&& \fFBS{0}& 12.7 & 10.7 & 2.4 & 2.3 & 20.1 & 4.9 & 38.1 & 11.9 & 3.7 & 29.8 & 10.3 & 8.9 \\
&& \fFBS{1}& 19.1 & 11.8 & 2.7 & 2.7 & 24.8 & 6.1 & 47.3 & 16.1 & 6.3 & 38.1 & 12.5 & 14.4 \\
&& \fFBS{\infty}& 36.1 & 12.2 & 2.7 & 2.9 & 28.1 & 6.9 & 81.9 & 18.4 & 11.1 & 46.3 & 15.4 & 19.4 \\
&& \FBS{}& 33.3 & 12.2 & 2.7 & 2.9 & 26.6 & 6.9 & 77.6 & 17.8 & 10.6 & 45.9 & 15.6 & 19.9 \\
\hline
\multirow{4}{*}{\rotatebox{90}{$\alpha$ = 0.01}}
&& \fFBS{0}& 17.3 & 12.4 & 5.3 & 4.9 & 25.8 & 10.3 & 50.0 & 17.2 & 7.7 & 44.1 & 13.7 & 16.9 \\
&& \fFBS{1}& 26.1 & 14.4 & 6.5 & 7.0 & 33.0 & 14.9 & 68.2 & 24.0 & 11.9 & 59.2 & 17.6 & 22.0 \\
&& \fFBS{\infty}& 47.7 & 15.5 & 6.7 & 8.6 & 39.8 & 20.8 & 86.9 & 34.2 & 12.0 & 72.0 & 21.9 & 23.6 \\
&& \FBS{}& 46.5 & 15.6 & 6.4 & 8.9 & 39.1 & 20.3 & 81.7 & 36.1 & 10.6 & 71.5 & 20.8 & 22.4 \\
\hline
\multirow{4}{*}{\rotatebox{90}{$\alpha$ = 0.05}}
&& \fFBS{0}& 21.8 & 15.7 & 16.1 & 11.5 & 33.2 & 19.6 & 65.0 & 24.0 & 12.5 & 67.9 & 19.6 & 23.5 \\
&& \fFBS{1}& 31.3 & 20.0 & 23.2 & 17.4 & 45.5 & 30.3 & 88.6 & 36.6 & 12.9 & 76.0 & 23.1 & 25.2 \\
&& \fFBS{\infty}& 72.6 & 27.7 & 30.4 & 32.4 & 72.9 & 101.0 & 88.4 & 80.8 & 12.9 & 79.0 & 23.5 & 25.2 \\
&& \FBS{}& 69.5 & 26.4 & 28.8 & 30.1 & 65.8 & 79.0 & 81.7 & 72.8 & 10.6 & 75.2 & 21.5 & 23.2 \\
\hline
\multirow{4}{*}{\rotatebox{90}{$\alpha$ = 0.1}}
&& \fFBS{0}& 23.4 & 19.4 & 31.7 & 16.0 & 39.1 & 28.9 & 74.2 & 29.8 & 11.8 & 73.2 & 22.4 & 24.7 \\
&& \fFBS{1}& 35.2 & 26.7 & 45.9 & 26.0 & 56.7 & 44.9 & 90.5 & 48.6 & 11.8 & 79.4 & 23.5 & 25.4 \\
&& \fFBS{\infty}& 85.9 & 38.1 & 62.0 & 81.7 & 199.1 & 123.3 & 90.5 & 83.1 & 11.8 & 81.3 & 23.6 & 25.4 \\
&& \FBS{}& 80.7 & 36.5 & 62.7 & 84.0 & 167.6 & 118.9 & 81.7 & 77.2 & 10.6 & 76.2 & 21.7 & 23.5 \\
\hline
\multirow{4}{*}{\rotatebox{90}{$\gamma$ = def}}
&& \fFBS{0}& 14.0 & 11.0 & 2.1 & 2.1 & 19.4 & 4.2 & 33.2 & 8.7 & 1.9 & 19.9 & 8.2 & 4.1 \\
&& \fFBS{1}& 20.8 & 12.1 & 2.5 & 2.5 & 24.3 & 4.9 & 41.5 & 11.9 & 2.9 & 25.7 & 10.0 & 5.3 \\
&& \fFBS{\infty}& 39.3 & 12.4 & 2.5 & 2.6 & 27.6 & 4.9 & 58.5 & 14.0 & 4.0 & 31.4 & 10.8 & 5.4 \\
&& \FBS{}& 38.3 & 12.4 & 2.5 & 2.6 & 26.5 & 4.9 & 51.5 & 13.4 & 3.6 & 31.7 & 10.8 & 5.4 \\
\hline
\multirow{4}{*}{\rotatebox{90}{$\gamma$ = 1}}
&& \fFBS{0}& 10.1 & 10.0 & 1.9 & 0.9 & 15.6 & 3.9 & 27.9 & 6.9 & 1.1 & 14.1 & 5.7 & 2.1 \\
&& \fFBS{1}& 14.3 & 10.4 & 1.9 & 1.0 & 19.1 & 4.0 & 33.1 & 8.2 & 2.0 & 16.7 & 7.5 & 3.2 \\
&& \fFBS{\infty}& 33.0 & 10.4 & 1.9 & 1.0 & 21.9 & 4.0 & 39.9 & 8.6 & 2.1 & 20.2 & 8.8 & 3.6 \\
&& \FBS{}& 32.8 & 10.4 & 1.9 & 1.0 & 21.7 & 4.0 & 38.1 & 9.0 & 2.0 & 21.4 & 8.9 & 3.6 \\
\hline
\multirow{4}{*}{\rotatebox{90}{$\gamma$ = 0.5}}
&& \fFBS{0}& 12.2 & 10.4 & 2.0 & 1.9 & 18.9 & 4.1 & 32.9 & 8.4 & 2.6 & 22.7 & 9.2 & 5.4 \\
&& \fFBS{1}& 18.4 & 11.5 & 2.3 & 2.2 & 23.3 & 4.6 & 41.7 & 11.7 & 4.3 & 29.5 & 11.1 & 7.5 \\
&& \fFBS{\infty}& 40.9 & 11.9 & 2.3 & 2.2 & 26.9 & 4.7 & 57.8 & 13.5 & 9.8 & 35.5 & 12.8 & 11.1 \\
&& \FBS{}& 36.5 & 12.0 & 2.3 & 2.2 & 26.1 & 4.7 & 51.3 & 13.1 & 9.4 & 34.5 & 12.9 & 10.6 \\
\hline
\multirow{4}{*}{\rotatebox{90}{$\gamma$ = 0}}
&& \fFBS{0}& 14.6 & 11.1 & 4.1 & 4.7 & 23.1 & 7.4 & 43.7 & 13.1 & 10.5 & 40.7 & 14.5 & 16.3 \\
&& \fFBS{1}& 22.0 & 12.2 & 4.7 & 6.4 & 28.9 & 10.8 & 56.5 & 17.6 & 12.4 & 51.6 & 18.7 & 22.3 \\
&& \fFBS{\infty}& 40.1 & 12.8 & 4.8 & 7.9 & 34.1 & 12.9 & 85.2 & 21.9 & 12.4 & 67.1 & 22.3 & 23.8 \\
&& \FBS{}& 38.9 & 12.7 & 4.8 & 7.9 & 34.0 & 12.9 & 81.7 & 20.5 & 10.6 & 67.4 & 20.9 & 22.4 \\
\hline
  \end{tabular}
\end{table*}

% Running time
\setlength{\tabcolsep}{.25em}
{\renewcommand{\arraystretch}{1.2}
\begin{table*}[!ht]
\centering
	\caption{Running time in seconds taken by each feature selection, selection criterion and hyper-parameter value. LASSO-FS was run with 100 and 1000 $\lambda$ values.}
    \label{tbl:timing}
  \scriptsize
  \begin{tabular}{ccl|*{12}{r}}
    \toprule
    &&
    Algorithm &
    \text{musk} &
    \text{sylva}&
    \text{madelon} &
    \text{secom} &
    \text{gina} &
    \text{hiva} &
    \text{gisette} &
    \text{p53} &
    \text{arcene} &
    \text{nova} &
    \text{dexter} &
    \text{dorothea} 
\\
\midrule
\multirow{4}{*}{\rotatebox{90}{$\alpha$ = 0.001}}
&& \fFBS{0}& 1.9 & 4.1 & 0.2 & 0.3 & 6.3 & 3.4 & 68.2 & 58.5 & 3.1 & 25.1 & 6.6 & 161.7 \\
&& \fFBS{1}& 3.9 & 9.5 & 0.7 & 0.7 & 14.3 & 10.8 & 217.3 & 163.5 & 7.5 & 110.0 & 24.7 & 452.1 \\
&& \fFBS{\infty}& 28.4 & 14.9 & 1.2 & 0.7 & 40.0 & 10.8 & 4191.7 & 429.8 & 39.5 & 787.4 & 87.4 & 2866.9 \\
&& \FBS{}& 54.3 & 44.9 & 1.5 & 1.7 & 179.1 & 26.7 & 6759.6 & 1343.2 & 84.7 & 4290.3 & 392.1 & 8906.9 \\
\hline
\multirow{4}{*}{\rotatebox{90}{$\alpha$ = 0.01}}
&& \fFBS{0}& 2.4 & 5.4 & 0.2 & 0.4 & 10.6 & 4.1 & 103.8 & 76.7 & 3.7 & 46.7 & 7.2 & 180.5 \\
&& \fFBS{1}& 6.4 & 11.2 & 0.7 & 0.9 & 21.4 & 14.2 & 305.0 & 194.8 & 16.9 & 148.8 & 31.1 & 563.9 \\
&& \fFBS{\infty}& 41.5 & 23.8 & 0.7 & 3.3 & 139.0 & 115.0 & 1705.8 & 2280.0 & 35.9 & 1175.9 & 221.0 & 1645.2 \\
&& \FBS{}& 111.9 & 75.0 & 3.2 & 5.7 & 296.2 & 181.7 & 12744.5 & 6066.6 & 84.6 & 10251.0 & 644.7 & 9059.1 \\
\hline
\multirow{4}{*}{\rotatebox{90}{$\alpha$ = 0.05}}
&& \fFBS{0}& 4.0 & 6.9 & 0.5 & 0.5 & 18.9 & 7.2 & 164.3 & 114.2 & 4.9 & 136.1 & 10.4 & 235.8 \\
&& \fFBS{1}& 7.7 & 14.0 & 1.1 & 1.3 & 43.3 & 19.8 & 482.5 & 291.8 & 20.5 & 447.6 & 53.5 & 785.0 \\
&& \fFBS{\infty}& 95.9 & 71.5 & 6.7 & 11.3 & 300.4 & 302.2 & 1370.5 & 12517.7 & 19.7 & 1449.2 & 102.4 & 1353.3 \\
&& \FBS{}& 232.5 & 134.8 & 24.0 & 20.2 & 1504.9 & 2385.9 & 12842.0 & 39944.9 & 85.4 & 11693.2 & 793.8 & 8918.9 \\
\hline
\multirow{4}{*}{\rotatebox{90}{$\alpha$ = 0.1}}
&& \fFBS{0}& 5.0 & 9.4 & 1.0 & 0.8 & 26.1 & 12.6 & 252.5 & 149.3 & 5.9 & 199.9 & 15.8 & 282.3 \\
&& \fFBS{1}& 10.6 & 21.5 & 2.2 & 2.5 & 60.6 & 41.7 & 1014.7 & 429.2 & 20.4 & 507.6 & 66.4 & 831.4 \\
&& \fFBS{\infty}& 107.3 & 94.1 & 10.9 & 54.5 & 1157.6 & 320.5 & 1748.6 & 6828.9 & 20.4 & 1466.4 & 121.5 & 837.6 \\
&& \FBS{}& 308.5 & 302.6 & 60.1 & 37.8 & 4488.9 & 2379.0 & 12837.2 & 40057.6 & 86.2 & 12423.0 & 795.4 & 9482.3 \\
\hline
\multirow{4}{*}{\rotatebox{90}{$\gamma$ = def}}
&& \fFBS{0}& 1.9 & 4.2 & 0.2 & 0.3 & 5.1 & 2.7 & 49.0 & 47.5 & 2.7 & 15.8 & 6.3 & 143.2 \\
&& \fFBS{1}& 4.9 & 9.5 & 0.7 & 0.6 & 13.0 & 6.2 & 115.3 & 136.7 & 6.3 & 82.1 & 24.4 & 349.6 \\
&& \fFBS{\infty}& 33.3 & 15.1 & 1.2 & 1.0 & 30.5 & 9.6 & 503.6 & 590.0 & 10.0 & 414.8 & 114.3 & 1258.2 \\
&& \FBS{}& 92.6 & 44.9 & 1.5 & 1.7 & 179.8 & 19.0 & 3938.1 & 918.3 & 37.4 & 2608.6 & 138.4 & 2290.1 \\
\hline
\multirow{4}{*}{\rotatebox{90}{$\gamma$ = 1}}
&& \fFBS{0}& 1.3 & 3.3 & 0.2 & 0.3 & 3.0 & 2.3 & 35.2 & 40.5 & 2.5 & 13.5 & 6.2 & 143.4 \\
&& \fFBS{1}& 2.5 & 8.3 & 0.7 & 0.6 & 10.1 & 5.8 & 137.2 & 109.3 & 5.8 & 65.4 & 20.5 & 300.0 \\
&& \fFBS{\infty}& 34.8 & 8.2 & 0.7 & 0.6 & 48.4 & 5.8 & 827.4 & 256.1 & 13.5 & 456.5 & 92.6 & 508.9 \\
&& \FBS{}& 71.9 & 34.4 & 1.1 & 0.6 & 120.1 & 15.5 & 2181.4 & 596.3 & 16.8 & 1129.6 & 137.8 & 817.1 \\
\hline
\multirow{4}{*}{\rotatebox{90}{$\gamma$ = 0.5}}
&& \fFBS{0}& 1.7 & 3.6 & 0.2 & 0.3 & 4.8 & 2.6 & 48.4 & 46.3 & 2.8 & 17.0 & 6.3 & 143.6 \\
&& \fFBS{1}& 3.8 & 8.8 & 0.7 & 0.6 & 12.5 & 6.1 & 114.7 & 133.7 & 6.5 & 85.3 & 24.6 & 413.4 \\
&& \fFBS{\infty}& 29.5 & 14.1 & 0.7 & 1.4 & 28.8 & 9.5 & 501.7 & 353.9 & 41.5 & 515.4 & 117.2 & 1041.8 \\
&& \FBS{}& 81.2 & 44.9 & 1.0 & 1.7 & 144.9 & 19.0 & 3917.7 & 918.1 & 84.6 & 2715.1 & 252.0 & 3277.8 \\
\hline
\multirow{4}{*}{\rotatebox{90}{$\gamma$ = 0}}
&& \fFBS{0}& 2.1 & 4.4 & 0.2 & 0.4 & 8.2 & 3.5 & 78.9 & 61.9 & 3.9 & 37.6 & 7.2 & 155.7 \\
&& \fFBS{1}& 5.4 & 9.8 & 0.7 & 0.9 & 17.7 & 12.4 & 250.4 & 169.0 & 14.7 & 139.0 & 30.9 & 561.0 \\
&& \fFBS{\infty}& 43.6 & 15.3 & 1.2 & 3.1 & 70.1 & 31.8 & 2573.7 & 594.8 & 30.3 & 1125.2 & 221.3 & 1641.6 \\
&& \FBS{}& 92.9 & 50.6 & 2.6 & 4.5 & 187.7 & 90.9 & 12762.1 & 1461.7 & 84.7 & 9702.6 & 645.5 & 8955.9 \\
\hline
&& LASSO-FS$^{100}$& 10.9 & 10.7 & 3.7 & 6.3 & 14.6 & 15.4 & 8.3 & 56.5 & 0.2 & 2.8 & 0.7 & 9.4 \\
&& LASSO-FS$^{1000}$& 11.1 & 23.8 & 3.9 & 8.8 & 25.3 & 47.1 & 39.2 & 153.2 & 1.6 & 21.2 & 5.8 & 89.0 \\
\hline
  \end{tabular}
\end{table*}

\begin{figure}[t!]
\begin{subfigure}[t]{0.475\textwidth}
\centering
\includegraphics[width=\textwidth]{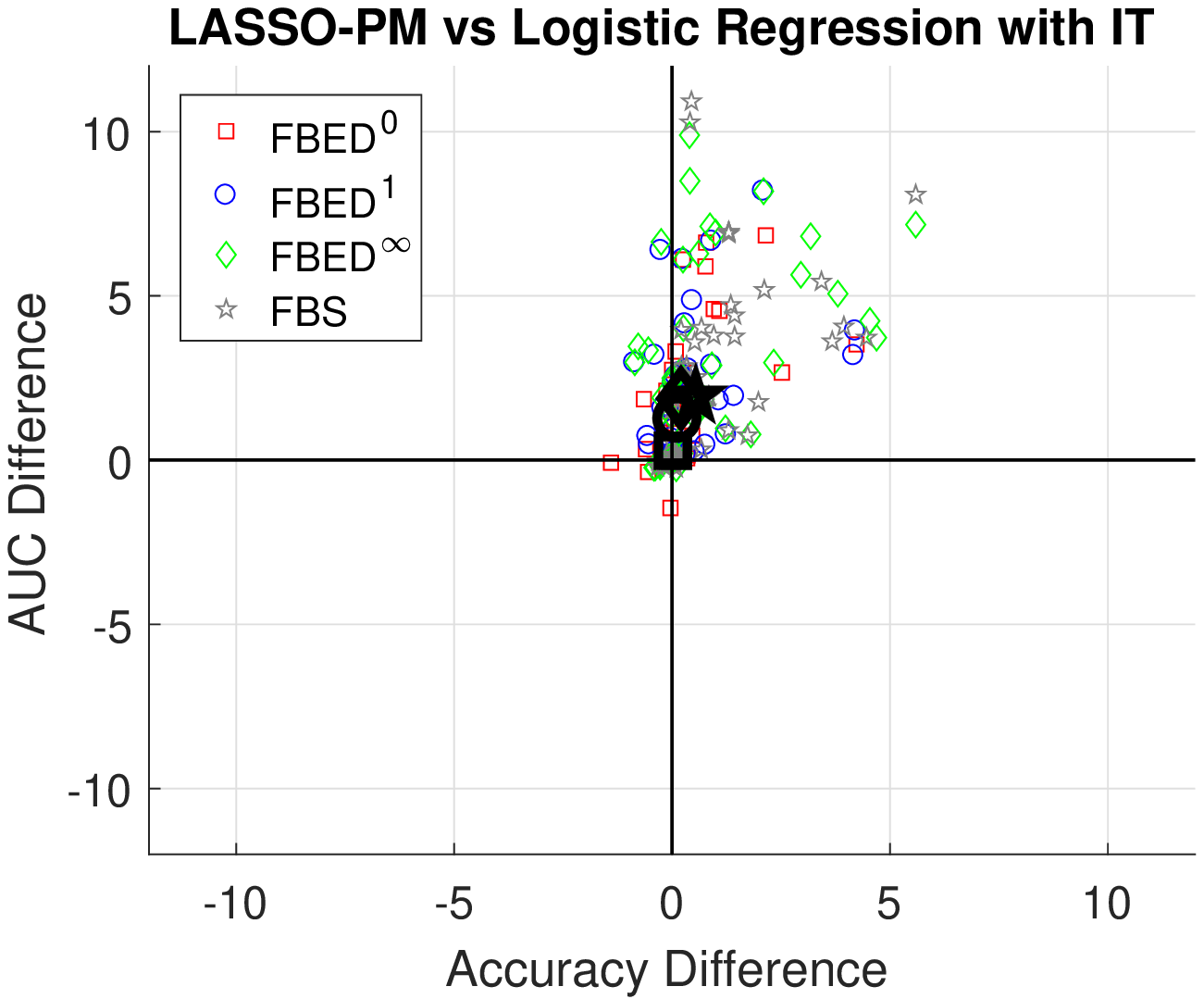}
\end{subfigure}
~
\begin{subfigure}[t]{0.475\textwidth}
\centering
\includegraphics[width=\textwidth]{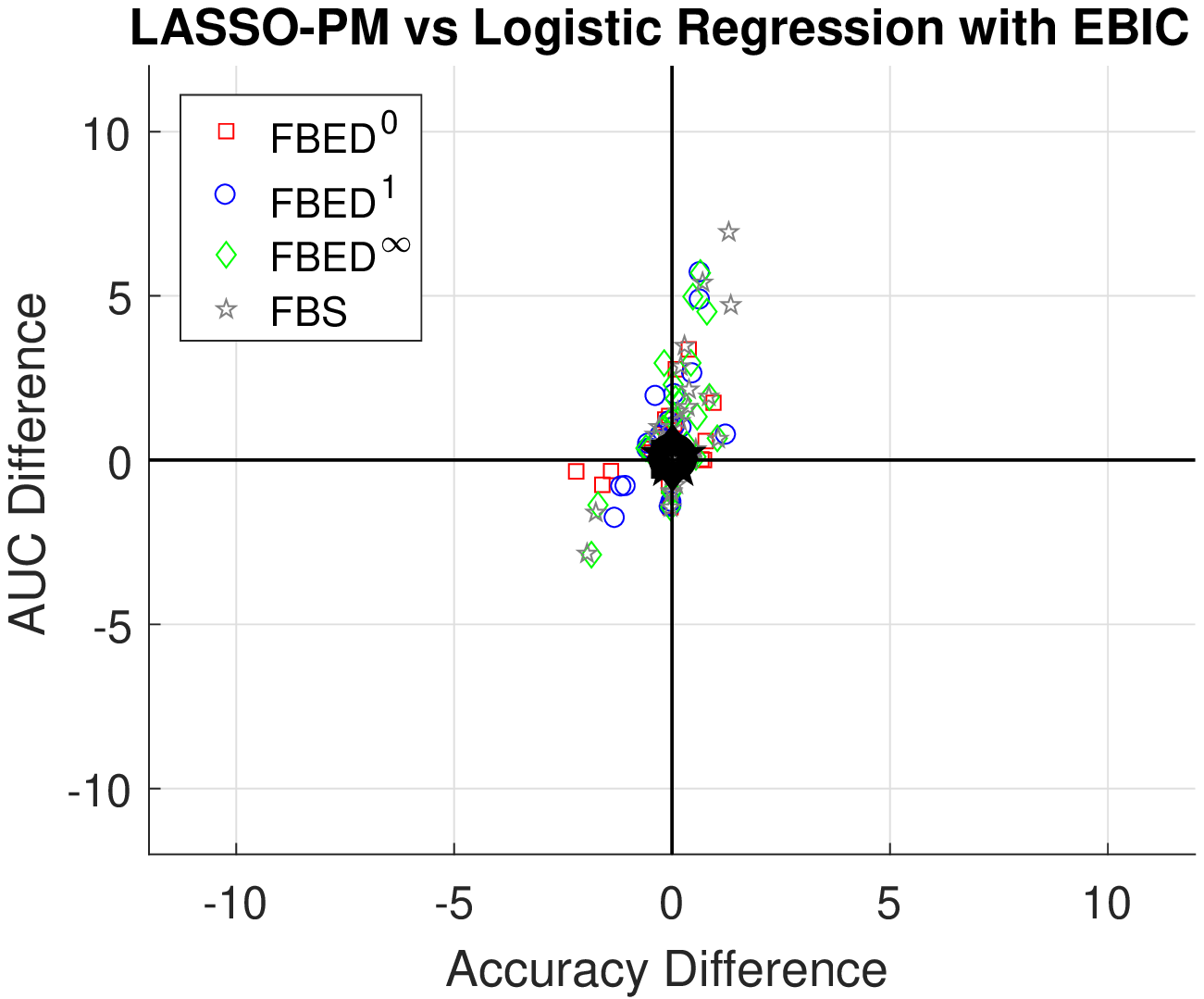}
\end{subfigure}
\caption{
\textbf{LASSO-PM vs Logistic Regression:} 
The x-axis and y-axis show the difference in accuracy and AUC respectively when using LASSO-PM and logistic regression as a predictive model, with positive values indicating that LASSO-PM performs better.
Black points show the median AUC and accuracy for each algorithm.
LASSO-PM performs at least as good as logistic regression, with LASSO-PM performing better when more variables are selected (\FBS{} or \fFBS{\infty} with IT).
}
\label{fig:l1vslr}
\end{figure}

\clearpage
\bibliographystyle{abbrv}
{
\small
\bibliography{ref}
}
\end{document}